%% file: iclr2018.tex
\documentclass{article} 
\usepackage{iclr2018_conference,times}
\usepackage[utf8]{inputenc}
\usepackage{hyperref}
\usepackage{url}

\usepackage{color}
\usepackage{bm}
\usepackage{amsmath, amsfonts}
\usepackage{amssymb}
\usepackage{mathtools}
\usepackage{amsthm}
\usepackage{mathrsfs}
\usepackage{algorithm}
\usepackage{algorithmic}
\usepackage{wrapfig}
\usepackage{nicefrac}
\usepackage{todonotes}
\usepackage{pgfplots}

\renewcommand{\vec}[1]{\bm{#1}}
\newcommand{\norm}[1]{\vert \vert#1\vert \vert}

\DeclareMathOperator*{\argmin}{\arg\!\min}

\newcommand{\prox}{\text{prox}}
\renewcommand{\L}{\mathcal{L}}

\newtheorem{proposition}{Proposition}

\newcommand{\bbR}{\mathbb{R}}
\newcommand{\textcirclednice}[1]{\raisebox{.3pt}{\textcircled{\raisebox{-.4pt} {#1}}}}

\newcommand{\Lip}{\mathscr{L}}

\renewcommand\cite{\citep}

\title{Proximal Backpropagation}

\author{Thomas Frerix${^1}$\thanks{ contributed equally} \ , Thomas M\"ollenhoff${^1}\footnotemark[1]$ \ , Michael Moeller${^2}\footnotemark[1]$ \ , Daniel Cremers${^1}$ \vspace{0.2cm} \\
    \texttt{thomas.frerix@tum.de}\\
    \texttt{thomas.moellenhoff@tum.de} \\
    \texttt{michael.moeller@uni-siegen.de} \\
    \texttt{cremers@tum.de}
    \vspace{0.2cm} \\
${^1}$ Technical University of Munich\\
${^2}$ University of Siegen
}

\iclrfinalcopy 

\begin{document}

\input{paper_content.tex}

\bibliographystyle{iclr2018_conference}
\bibliography{nips_2017}

\newpage
\section*{Appendix}
\input{supp_content.tex}

\end{document}

%% file: paper_content.tex
\maketitle

\begin{abstract}
We propose proximal backpropagation (ProxProp) as a novel algorithm that takes \textit{implicit} instead of explicit gradient steps to update the network parameters during neural network training.
Our algorithm is motivated by the step size limitation of explicit gradient descent, which poses an impediment for optimization.
ProxProp is developed from a general point of view on the backpropagation algorithm, currently the most common technique to train neural networks via stochastic gradient descent and variants thereof.
Specifically, we show that backpropagation of a prediction error is equivalent to sequential gradient descent steps on a quadratic penalty energy, which comprises the network activations as variables of the optimization.
We further analyze theoretical properties of ProxProp and in particular prove that the algorithm yields a descent direction in parameter space and can therefore be combined with a wide variety of convergent algorithms.
Finally, we devise an efficient numerical implementation that integrates well with popular deep learning frameworks.
We conclude by demonstrating promising numerical results and show that ProxProp can be effectively combined with common first order optimizers such as Adam. 
\end{abstract}

\section{Introduction}
In recent years neural networks have gained considerable attention in
solving difficult correlation tasks such as classification in computer
vision \cite{Krizhevsky2012} or sequence learning \cite{Sutskever2014}
and as building blocks of larger learning systems \cite{Silver2016}.
Training neural networks is accomplished by optimizing a nonconvex, possibly nonsmooth, nested function of the network parameters. 
Since the introduction of stochastic gradient descent (SGD) \cite{Robbins1951,Bottou1991}, several more sophisticated optimization methods have been studied.
One such class is that of quasi-Newton methods, as for example the
comparison of L-BFGS with SGD in \cite{Le2011a}, Hessian-free
approaches \cite{martens2010deep}, and the Sum of Functions Optimizer in \cite{Sohl-Dickstein2014}.
Several works consider specific properties of energy landscapes of deep learning models such as frequent saddle points \cite{Dauphin2014} and well-generalizable local optima \cite{Chaudhari2016}. 
Among the most popular optimization methods in currently used deep learning frameworks are momentum based improvements of classical SGD, notably Nesterov's Accelerated Gradient \cite{Nesterov1983, Sutskever2013}, and the Adam optimizer \cite{Kingma2015a}, which uses estimates of first and second order moments of the gradients for parameter updates.

Nevertheless, the optimization of these models remains challenging, as learning with SGD and its variants requires careful weight initialization and a sufficiently small learning rate in order to yield a stable and convergent algorithm. 
Moreover, SGD often has difficulties in propagating a learning signal deeply into a network, commonly referred to as the vanishing gradient problem \cite{hochreiter2001gradient}.

Training neural networks can be formulated as a constrained optimization problem by explicitly introducing the network activations as variables of the optimization, which are coupled via layer-wise constraints to enforce a feasible network configuration.
The authors of \cite{Carreira-Perpinan2014} have tackled this problem with a quadratic penalty approach, the method of auxiliary coordinates (MAC). 
Closely related, \cite{Taylor2016} introduce additional auxiliary variables to further split linear and nonlinear transfer between layers and propose a primal dual algorithm for optimization.
From a different perspective, \cite{LeCun1988} takes a Lagrangian approach to formulate the constrained optimization problem. 

In this work, we start from a constrained optimization point of view on the classical backpropagation algorithm. 
We show that backpropagation can be interpreted as a method alternating between two steps. 
First, a forward pass of the data with the current network weights.
Secondly, an ordered sequence of gradient descent steps on a quadratic penalty energy.

Using this point of view, instead of taking explicit gradient steps to update the network parameters associated with the \textit{linear} transfer functions, we propose to use implicit gradient steps (also known as proximal steps, for the definition see \eqref{eq:proximalPointAlgorithm}). We prove that such a model yields a descent direction and can therefore be used in a wide variety of (provably convergent) algorithms under weak assumptions. Since an exact proximal step may be costly, we further consider a matrix-free conjugate gradient (CG) approximation, which can directly utilize the efficient pre-implemented forward and backward operations of any deep learning framework. We prove that this approximation still yields a descent direction and demonstrate the effectiveness of the proposed approach in PyTorch. 

\section{Model and notation}
We propose a method to train a general $L$-layer neural network of the functional form
\begin{equation}
  \begin{aligned}
J(\vec{\theta}; X, y) &= \L_y(\phi(\theta_{L-1}, \sigma(\phi(\cdots ,\sigma(\phi(\theta_1,X)) \cdots )).
\end{aligned}
\label{eq:network_loss}
\end{equation}
Here, $J(\vec{\theta}; X, y) $ denotes the training loss as a function of the network parameters $\vec{\theta}$, the input data $X$ and the training targets $y$.
As the final loss function $\L_y$ we choose the softmax cross-entropy for our classification experiments.
$\phi$ is a linear transfer function and $\sigma$ an elementwise nonlinear transfer function.
As an example, for fully-connected neural networks $\theta = (W,b)$ and $\phi(\theta,a) = Wa + b \mathbf{1}$.
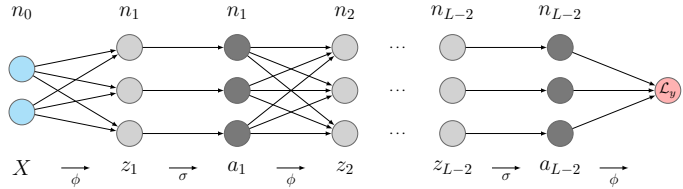
\begin{figure}[t!]
    \centering
    \resizebox{0.7\linewidth}{!}{\input{figures/networkDiagram.tex}}
  \caption{Notation overview. For an $L$-layer feed-forward network we denote the explicit layer-wise activation variables as $z_l$ and $a_l$. The transfer
    functions are denoted as $\phi$ and $\sigma$. Layer $l$ is of size~$n_l$.}
  \label{fig:network_diagram}
\end{figure}

While we assume the nonlinearities $\sigma$ to be continuously differentiable functions for analysis purposes, our numerical experiments indicate that the proposed scheme extends to rectified linear units (ReLU), $\sigma(x) = \max(0,x)$. 
Formally, the functions $\sigma$ and $\phi$ map between spaces of different dimensions depending on the layer. 
However, to keep the presentation clean, we do not state this dependence explicitly. Figure~\ref{fig:network_diagram} illustrates our notation for the fully-connected network architecture. 

Throughout this paper, we denote the Euclidean norm for vectors and the Frobenius norm for matrices by $\norm{\cdot}$, induced by an inner product $\langle \cdot, \cdot \rangle$. 
We use the gradient symbol $\nabla$ to denote the transpose of the Jacobian matrix, such that the chain rule applies in the form ``inner derivative times outer derivative''. For all computations involving matrix-valued functions and their gradient/Jacobian, we uniquely identify all involved quantities with their vectorized form by flattening matrices in a column-first order.
Furthermore, we denote by $A^*$ the adjoint of a linear operator $A$.

\section{Penalty formulation of backpropagation}
The gradient descent iteration on a nested function $J(\vec{\theta}; X, y)$,
\begin{equation}
\begin{aligned}
    \vec{\theta}^{k+1} = \vec{\theta}^{k} - \tau \nabla J(\vec{\theta}^{k}; X, y),
\label{eq:gradient_descent}
\end{aligned}
\end{equation}
is commonly implemented using the backpropagation algorithm \cite{Rumelhart1986}. 
As the basis for our proposed optimization method, we derive a connection between the classical backpropagation algorithm and quadratic penalty functions of the form
\begin{align}
\label{eq:penaltyFunction}
\begin{split}
E(\vec{\theta}, \vec{a}, \vec{z}) &= \L_y(\phi(\theta_{L-1}, a_{L-2}))  
 + \sum_{l=1}^{L-2} \frac{\gamma}{2} \|\sigma(z_{l}) - a_l\|^2  
+ \frac{\rho}{2} \|\phi(\theta_l, a_{l-1}) - z_l \|^2.
\end{split}
\end{align} 
The approach of \cite{Carreira-Perpinan2014} is based on the minimization of \eqref{eq:penaltyFunction}, as 
under mild conditions the limit $\rho, \gamma \rightarrow \infty$ leads to the convergence of the sequence of minimizers of $E$ to the minimizer of $J$ \cite[Theorem~17.1]{Nocedal2006}. 
In contrast to~\cite{Carreira-Perpinan2014} we do not optimize \eqref{eq:penaltyFunction}, but rather use a connection of \eqref{eq:penaltyFunction} to the classical backpropagation algorithm 
to motivate a semi-implicit optimization algorithm for the original cost function $J$. 

Indeed, the iteration shown in Algorithm~\ref{alg:penalty_backprop} of forward passes followed by a sequential gradient descent on the penalty function $E$ is equivalent to the classical gradient descent iteration.
\begin{proposition}
Let $\mathcal{L}_y$, $\phi$ and $\sigma$ be continuously differentiable. For $\rho = \gamma = 1 / \tau$ and $ \vec{\theta}^k$ as the input to Algorithm~\ref{alg:penalty_backprop}, its output $ \vec{\theta}^{k+1}$ satisfies \eqref{eq:gradient_descent}, i.e., Algorithm \ref{alg:penalty_backprop} computes one gradient descent iteration on $J$.
\label{prop:equivalence}
\end{proposition}
\begin{proof}
For this and all further proofs we refer to Appendix~\ref{sec:proofs}. 
\end{proof}

\begin{figure}
\newcommand{\Comment}[1]{\textit{// #1}}
\newcommand{\LineComment}[1]{\hfill\textit{// #1}}
\begin{minipage}[t!]{6.9cm}
  \vspace{0pt}
  \begin{algorithm}[H]
    \small
    \caption{- Penalty formulation of BackProp}
    \label{alg:penalty_backprop}
    \begin{algorithmic}
      \STATE {\bfseries Input:} Current parameters $\vec{\theta}^k$.\\[1mm]
      \STATE \Comment{Forward pass.}
      \FOR{$l=1$ {\bfseries to} $L-2$}
      \STATE $z_{l}^{k} = \phi(\theta_{l}^{k}, a_{l-1}^{k})$, \LineComment{$a_0 = X$.}
      \STATE $a_{l}^{k} = \sigma(z_l^{k})$.
      \ENDFOR\\[1mm]
      \STATE \Comment{Perform minimization steps on \eqref{eq:penaltyFunction}.}
      \STATE \textcirclednice{a} grad. step on $E$ wrt. $(\theta_{L-1}, a_{L-2})$ \label{alg2:minstep_a}
      \FOR{$l=L-2$ {\bfseries to} $1$}
      \STATE \textcirclednice{b} grad. step on $E$ wrt. $z_{l}$ and $a_{l-1}$, \label{alg2:minstep_b} 
      \STATE {\color{purple}\textcirclednice{c} grad. step on $E$ wrt. $\theta_l$. \label{alg2:minstep_c}}
      \ENDFOR\\[1mm]
      \STATE {\bfseries Output:} New parameters $\vec{\theta}^{k + 1}$.
    \end{algorithmic}
  \end{algorithm}
\end{minipage}
\begin{minipage}[t!]{6.9cm}
  \vspace{0pt}
  \begin{algorithm}[H]
    \small
    \caption{- ProxProp}
    \label{alg:proxprop}
    \begin{algorithmic}
      \STATE {\bfseries Input:} Current parameters $\vec{\theta}^k$.\\[1mm]
      \STATE \Comment{Forward pass.}
      \FOR{$l=1$ {\bfseries to} $L-2$}
      \STATE $z_{l}^{k} = \phi(\theta_{l}^{k}, a_{l-1}^{k})$, \LineComment{$a_0 = X$.}
      \STATE $a_{l}^{k} = \sigma(z_l^{k})$.
      \ENDFOR\\[1mm]
      \STATE \Comment{Perform minimization steps on \eqref{eq:penaltyFunction}.}
      \STATE \textcirclednice{a} grad. step on $E$ wrt. $(\theta_{L-1}, a_{L-2})$, Eqs.~\ref{eq:update_a_explicit},~\ref{eq:update_theta_explicit}. \label{alg2:minstep_a}
      \FOR{$l=L-2$ {\bfseries to} $1$}
      \STATE \textcirclednice{b} grad. step on $E$ wrt. $z_{l}$ and $a_{l-1}$, Eqs.~\ref{eq:update_z},~\ref{eq:update_a}. \label{alg2:minstep_b} 
      \STATE {\color{purple}\textcirclednice{c} prox step on $E$ wrt. $\theta_l$, Eq.~\ref{eq:wba_update}. \label{alg2:minstep_c}}
      \ENDFOR\\[1mm]
      \STATE {\bfseries Output:} New parameters $\vec{\theta}^{k + 1}$.
    \end{algorithmic}
  \end{algorithm}
\end{minipage}
\end{figure}

\section{Proximal backpropagation}
The interpretation of Proposition~\ref{prop:equivalence} leads to the natural idea of replacing the explicit gradient steps
 \hyperref[alg1:gradient_a]{\textcirclednice{a}}, \hyperref[alg1:gradient_b]{\textcirclednice{b}} and \hyperref[alg1:gradient_c]{\textcirclednice{c}} in Algorithm~\ref{alg:penalty_backprop} with other -- possibly more
 powerful -- minimization steps. 
 We propose Proximal Backpropagation (ProxProp) as one such algorithm that takes \textit{implicit} instead of \textit{explicit} gradient steps to update the network parameters $\theta$ in step {\textcirclednice{c}}.
This algorithm is motivated by the step size restriction of gradient descent.

\subsection{Gradient descent and proximal mappings}
\label{sec:proximal_mapping}
Explicit gradient steps pose severe restrictions on the allowed step size $\tau$: Even for a convex, twice continuously differentiable, $\Lip$-smooth function $f : \bbR^n \rightarrow \bbR$, the convergence of the gradient descent algorithm can only be guaranteed for step sizes $0 < \tau < 2/\Lip$. 
The Lipschitz constant $\Lip$ of the gradient $\nabla f$ is in this case equal to the largest eigenvalue of the Hessian $H$.
With the interpretation of backpropagation as in Proposition~\ref{prop:equivalence}, gradient steps are taken on quadratic functions.
As an example for the first layer,
\begin{align}
    f(\theta) = \frac{1}{2} \norm{\theta X - z_1}^2 \; .
    \label{eq:quadratic_example}
\end{align}
In this case the Hessian is $H = XX^\top$, which is often ill-conditioned. For the CIFAR-10 dataset the largest eigenvalue is $6.7 \cdot 10^6$, which is seven orders of magnitude larger than the smallest eigenvalue. 
Similar problems also arise in other layers where poorly conditioned matrices $a_l$ pose limitations for guaranteeing the energy $E$ to decrease.

The proximal mapping \cite{Moreau1965} of a function $f : \bbR^n \to
\bbR$ is defined as: 
\begin{equation}
  \begin{aligned}
    \prox_{\tau f}(y) := \argmin_{x \in \bbR^n} ~ f(x) + \frac{1}{2 \tau}
    \norm{x - y}^2.
  \end{aligned}
  \label{eq:proximalMapping}
\end{equation}
By rearranging the optimality conditions to \eqref{eq:proximalMapping} and taking $y = x^k$, it can be interpreted as an \emph{implicit} gradient step evaluated at the new point $x^{k+1}$ (assuming differentiability of $f$):
\begin{equation}
  \begin{aligned}
    x^{k+1} &:= \argmin_{x \in \bbR^n} ~ f(x) + \frac{1}{2 \tau} \norm{x - x^k}^2 = x^k - \tau \nabla f(x^{k+1}).
    \label{eq:proximalPointAlgorithm}
  \end{aligned}
\end{equation}
The iterative algorithm~\eqref{eq:proximalPointAlgorithm} is known as the proximal point algorithm \cite{Martinet1970}.
In contrast to explicit gradient descent this algorithm is \textit{unconditionally stable}, i.e. the update scheme \eqref{eq:proximalPointAlgorithm} monotonically decreases $f$ for any $\tau>0$, since it holds by definition of the minimizer $x^{k+1}$ that $f(x^{k+1}) + \frac{1}{2 \tau} \norm{x^{k+1} - x^k}^2 \leq f(x^k)$. 

Thus proximal mappings yield unconditionally stable subproblems in the following sense: The update in $\theta_l$ provably decreases the penalty
energy $E(\vec{\theta}, \vec{a}^k, \vec{z}^k)$ from \eqref{eq:penaltyFunction} for fixed activations $(\vec{a}^k, \vec{z}^k)$ for any choice of step size. This motivates us to use proximal steps as
depicted in Algorithm~\ref{alg:proxprop}. 

\subsection{ProxProp}
We propose to replace explicit gradient steps with proximal steps to update the network parameters of the linear transfer function. 
More precisely, after the forward pass
\begin{align}
\label{eq:forwardPass}
\begin{split}
z_{l}^{k} =&~ \phi(\theta_{l}^{k}, a_{l-1}^{k}), \\
a_l^k =&~ \sigma(z_l^k),
\end{split}
\end{align}
we keep the explicit gradient update equations for $z_l$ and $a_l$.
The last layer update is 
\begin{align}
    a_{L-2}^{k + \nicefrac{1}{2}} =&~ a_{L-2}^k - \tau \nabla_{a_{L-2}}\L_y(\phi(\theta_{L-1}, a_{L-2})),
    \label{eq:update_a_explicit}
\end{align}
and for all other layers, 
 \begin{align}
 \label{eq:update_z}
  z_l^{k+\nicefrac{1}{2}} =&~ z_l^k -  \sigma'(z_l^k)(\sigma(z_l^k) - a_l^{k + \nicefrac{1}{2}}),\\ 
  \label{eq:update_a}
  a_{l-1}^{k + \nicefrac{1}{2}} =&~ a_{l-1}^k -  \nabla \left( \frac{1}{2} \|\phi(\theta_l, \cdot) - z_l^{k+\nicefrac{1}{2}} \|^2\right)(a_{l-1}^k),
\end{align}
where we use $ a_{l}^{k+\nicefrac{1}{2}}$ and $ z_{l}^{k+\nicefrac{1}{2}}$ to denote the updated variables before the forward pass of the next iteration and multiplication in \eqref{eq:update_z} is componentwise.
However, instead of taking explicit gradient steps to update the linear transfer parameters $\theta_l$, we take proximal steps
\begin{align}
    \theta_{l}^{k+1} = \argmin_{\theta} ~ \frac{1}{2} \norm{\phi(\theta, a_{l-1}^k) - z_l^{k+\nicefrac{1}{2}}}^2 + \frac{1}{2 \tau_\theta} \norm{\theta - \theta_{l}^{k}}^2. 
    \label{eq:wba_update}
\end{align}
This update can be computed in closed form as it amounts to a linear solve (for details see Appendix~\ref{sec:prox_linear}).
While in principle one can take a proximal step on the final loss $\L_y$, for efficiency reasons we choose an explicit gradient step, as the proximal step 
does not have a closed form solution in many scenarios (e.g. the softmax cross-entropy loss in classification problems).
Specifically,
\begin{align}
    \theta_{L-1}^{k+1} =&~ \theta_{L-1}^k - \tau \nabla_{\theta_{L-1}}\L_y(\phi(\theta_{L-1}^k, a_{L-2}^k)). \label{eq:update_theta_explicit}
\end{align}
Note that we have eliminated the step sizes in the updates for $z_l$ and $a_{l-1}$ in \eqref{eq:update_z} and \eqref{eq:update_a}, as such updates correspond to the choice of $\rho = \gamma = \frac{1}{\tau}$ in the penalty function \eqref{eq:penaltyFunction} and are natural in the sense of Proposition \ref{prop:equivalence}. 
For the proximal steps in the parameters $\theta$ in \eqref{eq:wba_update} we have introduced a step size $\tau_\theta$ which -- as we will see in Proposition \ref{prop:characterizingProxProp} below -- changes the descent metric opposed to $\tau$ which rather rescales the magnitude of the update. 

We refer to one sweep of updates according to equations \eqref{eq:forwardPass}~-~\eqref{eq:update_theta_explicit} as \textit{ProxProp}, as it closely resembles the classical backpropagation (BackProp), but replaces the parameter update by a proximal mapping instead of an explicit gradient descent step. In the following subsection we analyze the convergence properties of ProxProp more closely.

\subsubsection{Convergence of ProxProp}
ProxProp inherits all convergence-relevant properties from the classical backpropagation algorithm, despite replacing explicit gradient steps with proximal steps: It minimizes the original network energy $J(\vec\theta; X, y)$ as its fixed-points are stationary points of $J(\vec\theta; X, y)$, and the update direction $\vec{\theta}^{k+1}-\vec{\theta}^k$ is a descent direction such that it converges when combined with a suitable optimizer.
In particular, it is straight forward to combine ProxProp with popular optimizers such as Nesterov's accelerated gradient descent~\cite{Nesterov1983} or Adam~\cite{Kingma2015a}.

In the following, we give a detailed analysis of these properties.
\begin{proposition}
\label{prop:characterizingProxProp}
For $l=1,\hdots, L-2$, the update direction $\vec{\theta}^{k+1}-\vec{\theta}^k$ computed by ProxProp meets
\begin{equation}
 \theta_{l}^{k+1}-\theta_l^k = -\tau\left(\frac{1}{\tau_\theta}I + (\nabla \phi(\cdot, a_{l-1}^k)) (\nabla \phi(\cdot, a_{l-1}^k))^*\right)^{-1}\nabla_{\theta_l} J(\vec{\theta}^k; X,y).
\end{equation}
\end{proposition}

In other words, ProxProp multiplies the gradient $\nabla_{\theta_l} J$ with the inverse of the positive definite, symmetric matrix
\begin{align}
\label{eq:matrixM}
M_l^k := \frac{1}{\tau_\theta}I + (\nabla \phi(\cdot,a_{l-1}^k)) (\nabla \phi(\cdot,a_{l-1}^k))^*,
\end{align}
which depends on the activations $a_{l-1}^k$ of the forward pass. Proposition \ref{prop:characterizingProxProp} has some important implications:
\begin{proposition}
\label{prop:stationaryPointsProxProp}
    For any choice of $\tau>0$ and $\tau_\theta>0$, fixed points $\vec \theta^*$ of ProxProp are stationary points of the original energy $J(\vec \theta; X, y)$.
\end{proposition}

Moreover, we can conclude convergence in the following sense.
\begin{proposition}
\label{prop:descent}
The ProxProp direction $\vec{\theta}^{k+1}-\vec{\theta}^k$ is a descent direction. Moreover, under the (weak) assumption that the activations $a_l^k$ remain bounded, the angle $\alpha^k$ between $-\nabla_{\vec{\theta}} J(\vec{\theta}^k; X,y)$ and $\vec{\theta}^{k+1}-\vec{\theta}^k$ remains uniformly bounded away from $\pi/2$, i.e. 
\begin{equation} 
\cos(\alpha^k)>c\geq 0, \qquad \forall k \geq 0,
\end{equation}
 for some constant $c$. 
\end{proposition}

Proposition \ref{prop:descent} immediately implies convergence of a whole class of algorithms that depend only on a provided descent direction.
We refer the reader to \cite[Chapter 3.2]{Nocedal2006} for examples and more details. 

Furthermore, Proposition~\ref{prop:descent} states convergence for any minimization scheme in step \textcirclednice{c} of Algorithm~\ref{alg:proxprop} that induces a descent direction in parameter space and thus provides the theoretical basis for a wide range of neural network optimization algorithms.

Considering the advantages of proximal steps over gradient steps, it is tempting to also update the auxiliary variables $a$ and $z$ in an implicit fashion. 
This corresponds to a proximal step in \textcirclednice{b} of Algorithm~\ref{alg:proxprop}.
However, one cannot expect an analogue version of Proposition~\ref{prop:stationaryPointsProxProp} to hold anymore. 
For example, if the update of $a_{L-2}$ in \eqref{eq:update_a_explicit} is replaced by a proximal step, the propagated error does not correspond to the gradient of the loss function $\L_y$, but to the gradient of its Moreau envelope. 
Consequently, one would then minimize a different energy.
While in principle this could result in an optimization algorithm with, for example, favorable generalization properties, we focus on minimizing the original network energy in this work and therefore do not further pursue the idea of implicit steps on $a$ and $z$. 

\subsubsection{Inexact solution of proximal steps}
As we can see in Proposition \ref{prop:characterizingProxProp}, the ProxProp updates differ from vanilla gradient descent by the variable metric induced by the matrices $(M_l^k)^{-1}$ with $M_l^k$ defined in \eqref{eq:matrixM}. 
Computing the ProxProp update direction $v_l^k:= \frac{1}{\tau}(\theta_l^{k+1}- \theta_l^{k})$ therefore reduces to solving the linear equation
\begin{align}
\label{eq:linearSystemUpdateDirection}
M_l^k v_l^k = -\nabla_{\theta_l}J(\vec{\theta}^k;X,y),
\end{align}
which requires an efficient implementation.
We propose to use a conjugate gradient (CG) method, not only because it is one of the most efficient methods for iteratively solving linear systems in general, but also because it can be implemented \textit{matrix-free}: It merely requires the application of the linear operator $M_l^k$ which consists of the identity and an application of $(\nabla \phi(\cdot,a_{l-1}^k)) (\nabla \phi(\cdot,a_{l-1}^k))^*$. The latter, however, is preimplemented for many linear transfer functions $\phi$ in common deep learning frameworks, because $\phi(x,a_{l-1}^k) = (\nabla \phi(\cdot,a_{l-1}^k))^*(x)$ is nothing but a forward-pass in $\phi$, and $\phi^*(z,a_{l-1}^k) = (\nabla \phi(\cdot,a_{l-1}^k))(z)$ provides the gradient with respect to the parameters $\theta$ if $z$ is the backpropagated gradient up to that layer. Therefore, a CG solver is straight-forward to implement in any deep learning framework using the existing, highly efficient and high level implementations of $\phi$ and $\phi^*$. 
For a fully connected network $\phi$ is a matrix multiplication and for a convolutional network the convolution operation.

As we will analyze in more detail in Section \ref{sec:approximateSolutions}, we approximate the solution to \eqref{eq:linearSystemUpdateDirection} with a few CG iterations, as the advantage of highly precise solutions does not justify the additional computational effort in obtaining them. 
Using any number of iterations provably does not harm the convergence properties of ProxProp:
\begin{proposition}
\label{prop:descentCgApproximation}
The direction $\tilde{v}_l^k$ one obtains from approximating the solution $v_l^k$ of \eqref{eq:linearSystemUpdateDirection} with the CG method remains a descent direction for any number of iterations.
\end{proposition}

\subsubsection{Convergence in the stochastic setting}
While the above analysis considers only the full batch setting, we remark that convergence of ProxProp can also be guaranteed in the stochastic setting under mild assumptions. Assuming that the activations $a_l^k$ remain bounded (as in Proposition~\ref{prop:descent}), the eigenvalues of $(M_l^k)^{-1}$ are uniformly contained in the interval $[\lambda, \tau_\theta]$ for some fixed $\lambda > 0$. 
Therefore, our ProxProp updates fulfill Assumption~4.3 in \cite{Bottou2016}, presuming the classic stochastic gradient fulfills them.
This guarantees convergence of stochastic ProxProp updates in the sense of \cite[Theorem~4.9]{Bottou2016}, i.e. for a suitable sequence of diminishing step sizes.

\section{Numerical evaluation}
\label{sec:experiments}
ProxProp generally fits well with the API provided by modern deep learning frameworks, since it can be implemented as a network layer with a custom backward pass for the proximal mapping.
We chose PyTorch for our implementation\footnote{\url{https://github.com/tfrerix/proxprop}}. 
In particular, our implementation can use the API's GPU compute capabilities; all numerical experiments reported below were conducted on an NVIDIA Titan X GPU. 
To directly compare the algorithms, we used our own layer for either proximal or gradient update steps (cf. step \textcirclednice{c} in Algorithms~\ref{alg:penalty_backprop}~and \ref{alg:proxprop}).
A ProxProp layer can be seamlessly integrated in a larger network architecture, also with other parametrized layers such as BatchNormalization.

\begin{figure*}[t!]
    \centering
  \begin{tabular}{ccc}
    \resizebox{.31\linewidth}{!}{\input{figures/exact_vs_inexact_prox_mlp_plot.tex}} &
    \resizebox{.31\linewidth}{!}{\input{figures/inexact_prox_vs_sgd_mlp_plot.tex}} &
    \resizebox{.31\linewidth}{!}{\input{figures/exact_vs_inexact_prox_mlp_plot_val_acc.tex}} 
  \end{tabular}
  \caption{Exact and inexact solvers for ProxProp compared with BackProp. \textbf{Left:} A more precise solution of the proximal subproblem leads to overall faster
    convergence, while even a very inexact solution (only 3 CG iterations) already outperforms classical backpropagation. \textbf{Center~\&~Right:} While the run time is comparable between the
    methods, the proposed ProxProp updates have better generalization performance ($\approx 54 \%$ for BackProp and $\approx 56 \%$ for ours on the test set). \vspace{-0.3cm}}
  \label{fig:exact_inexact_solver}
\end{figure*}
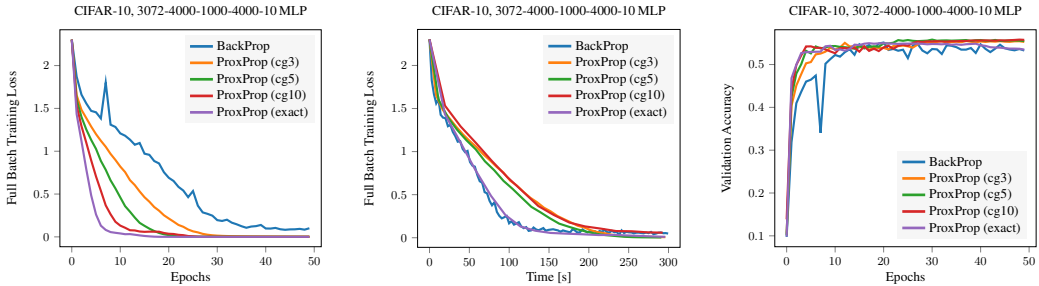

\subsection{Exact and approximate solutions to proximal steps}
\label{sec:approximateSolutions}
We study the behavior of ProxProp in comparison to classical BackProp for a supervised visual learning problem on the CIFAR-10 dataset.
We train a fully connected network with architecture $3072-4000-1000-4000-10$ and ReLU nonlinearities. As the loss function, we chose the cross-entropy between the probability distribution
obtained by a softmax nonlinearity and the ground-truth labels. We used a subset of 45000 images for training while keeping 5000 images as a validation set. 
We initialized the parameters $\theta_l$ uniformly in $[-1/\sqrt{n_{l-1}}, 1/\sqrt{n_{l-1}}]$, the default initialization of PyTorch. 

Figure~\ref{fig:exact_inexact_solver} shows the decay of the full batch training loss over epochs (left) and training time (middle) for a Nesterov momentum\footnote{PyTorch's Nesterov SGD for direction $\vec d(\vec \theta^k)$: $\vec m^{k+1} = \mu \vec m^k + \vec d(\vec\theta^k), \vec\theta^{k+1} = \vec\theta^k - \tau (\mu \vec m^{k+1} + \vec d(\vec \theta^k))$.} based optimizer using a momentum
of $\mu=0.95$ and minibatches of size 500. We used $\tau_\theta=0.05$ for the ProxProp variants along with $\tau = 1$. For BackProp we chose $\tau = 0.05$ as the optimal value we found in a grid search. 

As we can see in Figure~\ref{fig:exact_inexact_solver}, using implicit steps indeed improves the optimization progress per epoch. 
Thanks to powerful linear algebra methods on the GPU, 
the exact ProxProp solution is competitive with BackProp even in terms of runtime. 

The advantage of the CG-based approximations, however, is that they generalize to arbitrary linear transfer functions in a matrix-free manner, i.e. they are independent of whether the matrices $M_l^k$ can be formed efficiently. 
Moreover, the validation accuracies (right plot in Figure~\ref{fig:exact_inexact_solver}) suggest that these approximations have generalization advantages in comparison to BackProp as well as the exact ProxProp method. Finally, we found the exact solution to be significantly more sensitive to changes of $\tau_\theta$ than its CG-based approximations. We therefore focus on the CG-based variants of ProxProp in the following. In particular, we can eliminate the hyperparameter $\tau_\theta$ and consistently chose $\tau_\theta = 1$ for the rest of this paper, while one can in principle perform a hyperparameter search just as for the learning rate $\tau$. Consequently, there are no additional parameters compared with BackProp.

\subsection{Stability for larger step sizes}
We compare the behavior of ProxProp and BackProp for different step sizes. 
Table~\ref{tab:stepSizeStability} shows the final full batch training loss after 50 epochs with various $\tau$. 
The ProxProp based approaches remain stable over a significantly larger range of $\tau$. 
Even more importantly, deviating from the optimal step size $\tau$ by one order of magnitude resulted in a divergent algorithm for classical BackProp, but still provides reasonable training results for ProxProp with 3 or 5 CG iterations. 
These results are in accordance with our motivation developed in Section~\ref{sec:proximal_mapping}.
From a practical point of view, this eases hyperparameter search over $\tau$.
\begin{table}[h!]
\centering
\renewcommand{\arraystretch}{1.5}
\resizebox{1\linewidth}{!}{
\begin{tabular}{| r | c | c | c | c | c | c | c | c | c |}
  \hline			
$\tau $ & $50$& $10$ & $5$& $1$& $0.5$& $0.1$& $0.05$& $5 \cdot 10^{-3}$& $5 \cdot 10^{-4}$\\
  \hline
   BackProp &  -- & -- & -- & -- & -- & 0.524 & 0.091 & 0.637& 1.531 \\
  ProxProp (cg1) &  77.9 & 0.079 &0.145 & 0.667 &  0.991 & 1.481&1.593&1.881& 2.184\\
    ProxProp (cg3) &  94.7 & 0.644 & 0.031 & $2\cdot 10^{-3}$ & 0.012 & 1.029 & 1.334 &1.814 &2.175  \\
    ProxProp (cg5) &  66.5 & 0.190 &0.027& $3\cdot 10^{-4}$ &$2\cdot 10^{-3}$& 0.399 & 1.049&1.765&2.175 \\
  \hline  
\end{tabular}
}
\caption{Full batch loss for conjugate gradient versions of ProxProp and BackProp after training for 50 epochs, where ``--'' indicates that the algorithm diverged to {\tt NaN}. The implicit ProxProp algorithms remain stable for a significantly wider range of step sizes.}
\label{tab:stepSizeStability}
\end{table}

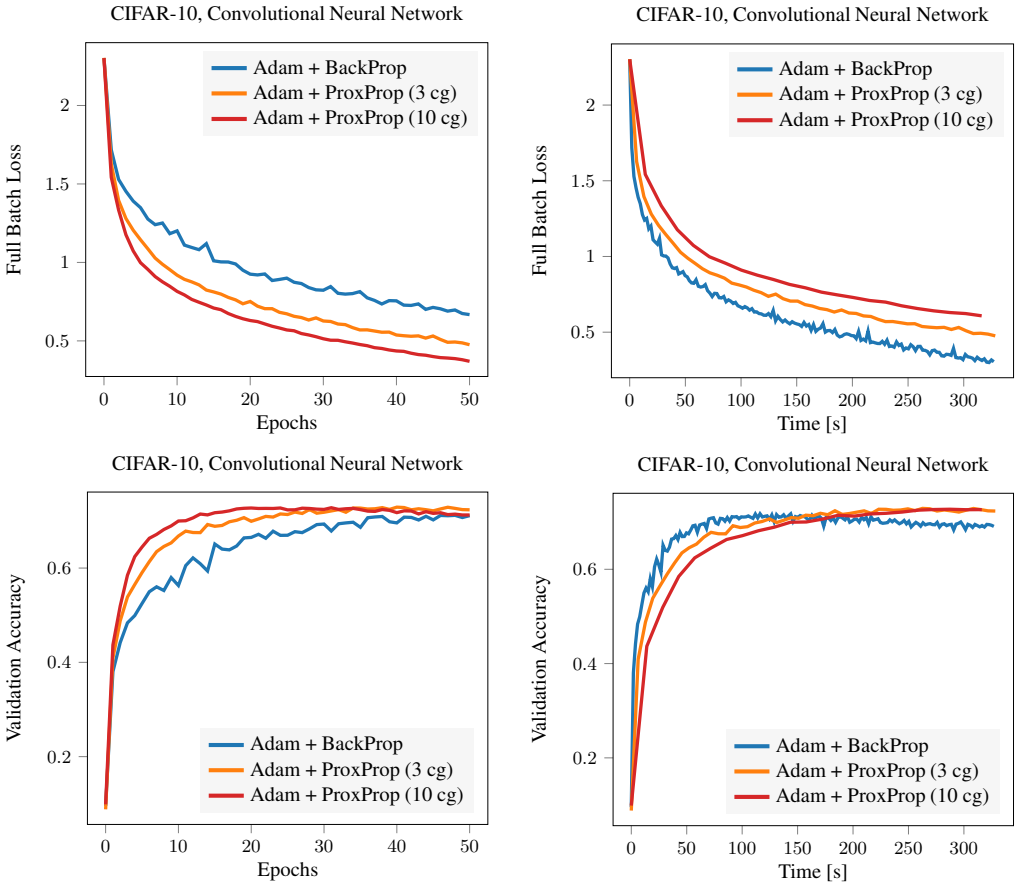
\begin{figure*}[t!]
    \centering
  \begin{tabular}{cc}
    \resizebox{.47\linewidth}{!}{\input{figures/proxprop_vs_sgd_adam_convnet_epochs_plot.tex}} &
    \resizebox{.47\linewidth}{!}{\input{figures/proxprop_vs_sgd_adam_convnet_time_plot.tex}} \\
    \resizebox{.47\linewidth}{!}{\input{figures/proxprop_vs_sgd_adam_convnet_epochs_plot_val.tex}} &
    \resizebox{.47\linewidth}{!}{\input{figures/proxprop_vs_sgd_adam_convnet_time_plot_val.tex}} \\
  \end{tabular}
  \caption{ProxProp as a first-order oracle in combination with the Adam optimizer. The proposed method leads to faster decrease of the full batch loss in epochs and to an overall higher accuracy on the validation set. The plots on the right hand side show data for a fixed runtime, which corresponds to a varying number of epochs for the different optimizers.}
  \label{fig:adam}
\end{figure*}

\subsection{ProxProp as a first-order oracle}
We show that ProxProp can be used as a gradient oracle for first-order optimization algorithms. In this section, we consider Adam \cite{Kingma2015a}.
Furthermore, to demonstrate our algorithm on a generic architecture with layers commonly used in practice, we trained on a convolutional neural network of the form:
{
\small
\begin{align*}
  &\rm{Conv} [16\times 32\times 32] \rightarrow \rm{ReLU} \rightarrow \rm{Pool} [16\times 16\times 16] 
    \rightarrow \rm{Conv} [20\times 16\times 16] \rightarrow \rm{ReLU} \\ 
  &\rightarrow \rm{Pool} [20\times 8\times 8] 
    \rightarrow \rm{Conv} [20\times 8\times 8] \rightarrow \rm{ReLU} \rightarrow \rm{Pool} [20\times 4\times 4] 
    \rightarrow \rm{FC + Softmax} [10\times 1\times 1]
\end{align*}
}
\vspace{-0.5cm}

Here, the first dimension denotes the respective number of filters with kernel size $5 \times 5$ and max pooling downsamples its input by a factor of two. 
We set the step size $\tau = 10^{-3}$ for both BackProp and ProxProp. 
 
The results are shown in Fig.~\ref{fig:adam}. Using parameter update directions induced by ProxProp within Adam leads to a significantly faster decrease of the full batch training loss in epochs. 
While the running time is higher than the highly optimized backpropagation method, we expect that it can be improved through further engineering efforts. 
We deduce from Fig.~\ref{fig:adam} that the best validation accuracy ($72.9\%$) of the proposed method is higher than the one obtained with classical backpropagation
($71.7\%$). Such a positive effect of proximal smoothing on the generalization capabilities of deep networks is consistent with the observations of~\citet{Chaudhari2017}.
Finally, the accuracies on the test set after 50 epochs are $70.7\%$ for ProxProp and $69.6\%$ for BackProp which suggests 
that the proposed algorithm can lead to better generalization.

\section{Conclusion}
We have proposed proximal backpropagation (ProxProp) as an effective method for training neural networks.  
To this end, we first showed the equivalence of the classical backpropagation algorithm with an algorithm that alternates between sequential gradient steps on a quadratic penalty function and forward passes through the network. 
Subsequently, we developed a generalization of BackProp, which replaces explicit gradient steps with implicit (proximal) steps, and proved that such a scheme yields a descent direction, even if the implicit steps are approximated by conjugate gradient iterations. 
Our numerical analysis demonstrates that ProxProp is stable across various choices of step sizes and shows promising results when compared with common stochastic gradient descent optimizers.

We believe that the interpretation of error backpropagation as the alternation between forward passes and sequential minimization steps on a penalty functional provides a theoretical basis for the development of further learning algorithms.

%% file: figures/networkDiagram.tex
\def\layersep{2.5cm}
\tikzset{>=latex}

\begin{tikzpicture}[->, draw=black, node distance=\layersep]
    \tikzstyle{every pin edge}=[<-]
    \tikzstyle{neuron}=[circle, draw=black!60, fill=black!25,minimum size=17pt,inner sep=0pt]
    \tikzstyle{input neuron}=[neuron, fill=cyan!30];
    \tikzstyle{output neuron}=[neuron, fill=red!30];
    \tikzstyle{linear activations}=[neuron, fill=lightgray!70];
    \tikzstyle{nonlinear activations}=[neuron, fill=darkgray!70];
    \tikzstyle{annot} = [text width=4.5em, text centered, font=\Large]
    
    \foreach \name / \y in {1,...,2}
        \node[input neuron] (I-\name) at (0,-\y) {};

    \foreach \name / \y in {1,...,3}
        \path[yshift=0.5cm]
            node[linear activations] (z1-\name) at (\layersep,-\y cm) {};	    

    \foreach \name / \y in {1,...,3}
        \path[yshift=0.5cm]
            node[nonlinear activations] (a1-\name) at (2*\layersep,-\y cm) {};

    \foreach \name / \y in {1,...,3}
        \path[yshift=0.5cm]
            node[linear activations] (z2-\name) at (3*\layersep,-\y cm) {};

    \node at (3.5*\layersep,-0.5 cm) {\ldots};
    \node at (3.5*\layersep,-1.5 cm) {\ldots};
    \node at (3.5*\layersep,-2.5 cm) {\ldots};

    \foreach \name / \y in {1,...,3}
        \path[yshift=0.5cm]
            node[linear activations] (z3-\name) at (4*\layersep,-\y cm) {};

    \foreach \name / \y in {1,...,3}
        \path[yshift=0.5cm]
            node[nonlinear activations] (a2-\name) at (5*\layersep,-\y cm) {};

	    \node[output neuron, right of=a2-2] (O) {$\mathcal{L}_y$};

    \foreach \source in {1,...,2}
        \foreach \dest in {1,...,3}
            \path (I-\source) edge (z1-\dest);

    \foreach \neuron in {1,...,3}
            \path (z1-\neuron) edge (a1-\neuron);

    \foreach \neuron in {1,...,3}
            \path (z3-\neuron) edge (a2-\neuron);

    \foreach \source in {1,...,3}
        \foreach \dest in {1,...,3}
            \path (a1-\source) edge (z2-\dest);

    \foreach \source in {1,...,3}
        \path (a2-\source) edge (O);

    \node[annot,above of=I-1, node distance=1.3cm] {$n_0$};
    \node[annot,below of=I-2, node distance=1.3cm] (I-annot-below) {$X$};
    \node[annot,above of=z1-1, node distance=0.8cm] {$n_1$};
    \node[annot,below of=z1-3, node distance=0.8cm] (z1-annot-below) {$z_{1}$};
    \draw[<-] (z1-annot-below) -- node [below] {$\phi$} (I-annot-below);
    \node[annot,above of=a1-1, node distance=0.8cm] {$n_1$};
    \node[annot,below of=a1-3, node distance=0.8cm] (a1-annot-below) {$a_{1}$};
    \draw[<-] (a1-annot-below) -- node [below] {$\sigma$} (z1-annot-below);
    \node[annot,above of=z2-1, node distance=0.8cm] {$n_2$};
    \node[annot,below of=z2-3, node distance=0.8cm] (z2-annot-below) {$z_{2}$};
    \draw[<-] (z2-annot-below) -- node [below] {$\phi$} (a1-annot-below);
    \node[annot,above of=z3-1, node distance=0.8cm] {$n_{L-2}$};
    \node[annot,below of=z3-3, node distance=0.8cm] (z3-annot-below) {$z_{L-2}$};
    \node[annot,above of=a2-1, node distance=0.8cm] {$n_{L-2}$};
    \node[annot,below of=a2-3, node distance=0.8cm] (a2-annot-below) {$a_{L-2}$};
    \draw[<-] (a2-annot-below) -- node [below] {$\sigma$} (z3-annot-below);
    \node[annot,below of=O, node distance=1.8cm] (out-annot-below) {};
    \draw[<-] (out-annot-below) -- node [below] {$\phi$} (a2-annot-below);

\end{tikzpicture}

%% file: figures/exact_vs_inexact_prox_mlp_plot.tex
\begin{tikzpicture}

\definecolor{color0}{rgb}{0.12156862745098,0.466666666666667,0.705882352941177}
\definecolor{color1}{rgb}{1,0.498039215686275,0.0549019607843137}
\definecolor{color2}{rgb}{0.172549019607843,0.627450980392157,0.172549019607843}
\definecolor{color3}{rgb}{0.83921568627451,0.152941176470588,0.156862745098039}
\definecolor{color4}{rgb}{0.580392156862745,0.403921568627451,0.741176470588235}

\begin{axis}[
title={CIFAR-10, 3072-4000-1000-4000-10 MLP},
xlabel={Epochs},
ylabel={Full Batch Training Loss},
xmin=-2.45, xmax=51.45,
ymin=-0.114995932622874, ymax=2.4182358531115,
tick align=outside,
tick pos=left,
x grid style={lightgray!92.026143790849673!black},
y grid style={lightgray!92.026143790849673!black},
legend entries={{BackProp},{ProxProp (cg3)},{ProxProp (cg5)},{ProxProp (cg10)},{ProxProp (exact)}},
legend style={draw=none},
xticklabel style={font=\footnotesize},
legend cell align={left},
legend style={line width=2.0, fill=gray!7},
yticklabel style={font=\footnotesize}
]
\addlegendimage{no markers, color0}
\addlegendimage{no markers, color1}
\addlegendimage{no markers, color2}
\addlegendimage{no markers, color3}
\addlegendimage{no markers, color4}
\addplot [ultra thick, color0]
table {%
0 2.30293766127692
1 1.87150613201989
2 1.66074070400662
3 1.5573687142796
4 1.47103758653005
5 1.45457435713874
6 1.38314944903056
7 1.80114405817456
8 1.30979862213135
9 1.28325177033742
10 1.20849868721432
11 1.18099272118674
12 1.13519287109375
13 1.07568638390965
14 1.09633612434069
15 0.971193301677704
16 0.957933515972561
17 0.885656751526727
18 0.85729043285052
19 0.755138741599189
20 0.690172145101759
21 0.652272172768911
22 0.583903333875868
23 0.540888408488697
24 0.4636932230658
25 0.535121353467305
26 0.364401699105899
27 0.285737033685048
28 0.268733946482341
29 0.247612319389979
30 0.195278767910269
31 0.185930912693342
32 0.200534804165363
33 0.165840847873025
34 0.146960982431968
35 0.124408366779486
36 0.106361057195399
37 0.120535651014911
38 0.125460391740004
39 0.139144845803579
40 0.0979311404542791
41 0.0969393912702799
42 0.104551484146052
43 0.0889807978438007
44 0.0826893469111787
45 0.0874210705773698
46 0.0882575432459513
47 0.0933801763587528
48 0.0842611155162255
49 0.100602049794462
};
\addplot [ultra thick, color1]
table {%
0 2.30308895375994
1 1.65446132818858
2 1.48359440697564
3 1.38602616786957
4 1.2881861699952
5 1.19928237862057
6 1.12200529045529
7 1.05227045615514
8 0.968558493587706
9 0.898367754618327
10 0.821107988225089
11 0.756362900469038
12 0.662139523691601
13 0.604331546359592
14 0.530213640464677
15 0.462932740648588
16 0.410812931259473
17 0.348615672190984
18 0.305788951781061
19 0.252604829271634
20 0.214595197472307
21 0.177829329007202
22 0.139382360296117
23 0.117830156700479
24 0.0857456148498588
25 0.0645427049448093
26 0.0524868350062105
27 0.0351585152869423
28 0.0250834879568881
29 0.0184838843842347
30 0.0141538239498105
31 0.0115423769379656
32 0.00972150907748275
33 0.00872214230087896
34 0.00758464226188759
35 0.00691335886820323
36 0.00615135515626106
37 0.00567934691595534
38 0.00525119827749829
39 0.00492867043034898
40 0.00454187619292902
41 0.00429577072823627
42 0.00401718430738482
43 0.00379059158472551
44 0.00358528288133028
45 0.00342110658530146
46 0.00326048185945385
47 0.00309610874392092
48 0.00296860185658766
49 0.00285660897433344
};
\addplot [ultra thick, color2]
table {%
0 2.30288625823127
1 1.60462185144424
2 1.39983485274845
3 1.25313382413652
4 1.13212262524499
5 1.02723664310243
6 0.887100534306632
7 0.783470136589474
8 0.660186680157979
9 0.565432102150387
10 0.461201825075679
11 0.358489187558492
12 0.294064832230409
13 0.225273457831807
14 0.163590393380986
15 0.128789309908946
16 0.0961331327756246
17 0.0682628682090176
18 0.0458986883362134
19 0.0302620469282071
20 0.0207004744249086
21 0.0111343536267264
22 0.00665343485048248
23 0.00473132918640557
24 0.00397652867250145
25 0.00335493262391537
26 0.0029767535523408
27 0.00269440850242972
28 0.00245899599718137
29 0.00227421073036061
30 0.00209680038711263
31 0.00195149383119618
32 0.00184395111993783
33 0.0017095594201237
34 0.00161500778049231
35 0.00152510622024743
36 0.00144828336520327
37 0.00138168409725444
38 0.00131434510824167
39 0.00125293879666262
40 0.00119997196856679
41 0.00114867777964618
42 0.00110141454497352
43 0.0010612989969862
44 0.00101930061613934
45 0.000983147482232501
46 0.000948954510062726
47 0.000916861011905389
48 0.000886186821541438
49 0.000859376410436299
};
\addplot [ultra thick, color3]
table {%
0 2.30253341197968
1 1.53307735125224
2 1.29302298360401
3 1.0982457127836
4 0.893995380401611
5 0.705279644992616
6 0.538138543897205
7 0.368056487705972
8 0.26385682374239
9 0.176204272442394
10 0.131573138634364
11 0.106296145915985
12 0.0774215534329414
13 0.0750095270574093
14 0.0627510632077853
15 0.0587704913069804
16 0.0609494363268216
17 0.0602914764235417
18 0.0604016699310806
19 0.0452281988329358
20 0.0362110191956162
21 0.0336102886953288
22 0.0239559804710249
23 0.0198283394591676
24 0.00993198273030834
25 0.00629773588152602
26 0.002201221896232
27 0.00110247308039106
28 0.000581250137959917
29 0.000454772526023185
30 0.000394018534118206
31 0.000354861936325
32 0.000325021595926955
33 0.000301185777870059
34 0.000281084802797219
35 0.000264145850734268
36 0.000249530728938731
37 0.000236801168648526
38 0.00022541804921477
39 0.000215483665104128
40 0.000206196912525532
41 0.000197619163049644
42 0.000190186924656801
43 0.000183380931654635
44 0.000176865703704405
45 0.000170946693429465
46 0.000165447086490329
47 0.000160337384537949
48 0.000155515564175504
49 0.000150966728688218
};
\addplot [ultra thick, color4]
table {%
0 2.30259527895186
1 1.43528225421906
2 1.11779054138396
3 0.786926379468706
4 0.485072463750839
5 0.261642377575239
6 0.130827238659064
7 0.0811428907430834
8 0.0558163819420669
9 0.0474345992836687
10 0.0401589810020394
11 0.0311777283954951
12 0.0324679580931034
13 0.0248847088362608
14 0.0172563866712153
15 0.010575224366039
16 0.00600877708444993
17 0.00243414431396458
18 0.000889114638104931
19 0.000397883477853611
20 0.000304315038212937
21 0.000267249340605405
22 0.000246411302618475
23 0.000234762150244529
24 0.000227511617413256
25 0.000223986329284445
26 0.000223118294282661
27 0.000222882674521922
28 0.00022612783650402
29 0.000232400873533657
30 0.000238517317403522
31 0.000244372368453898
32 0.00025508913499329
33 0.000264494958294866
34 0.000276310496944158
35 0.000288486958339086
36 0.000301049677202375
37 0.000314103888296005
38 0.000329236473771743
39 0.000345682335965749
40 0.000365121841443599
41 0.000383882659500361
42 0.000396837754588988
43 0.000411206827266142
44 0.000433620845433325
45 0.000451331031379393
46 0.000472092838996711
47 0.000494845475703995
48 0.000521154022943746
49 0.000540319970524352
};
\end{axis}

\end{tikzpicture}

%% file: figures/inexact_prox_vs_sgd_mlp_plot.tex
\begin{tikzpicture}

\definecolor{color0}{rgb}{0.12156862745098,0.466666666666667,0.705882352941177}
\definecolor{color1}{rgb}{1,0.498039215686275,0.0549019607843137}
\definecolor{color2}{rgb}{0.172549019607843,0.627450980392157,0.172549019607843}
\definecolor{color3}{rgb}{0.83921568627451,0.152941176470588,0.156862745098039}
\definecolor{color4}{rgb}{0.580392156862745,0.403921568627451,0.741176470588235}

\begin{axis}[
title={CIFAR-10, 3072-4000-1000-4000-10 MLP},
xlabel={Time [s]},
ylabel={Full Batch Training Loss},
xmin=-14.9780244231224, xmax=314.538512885571,
ymin=-0.111631768432886, ymax=2.41807565481674,
tick align=outside,
tick pos=left,
x grid style={lightgray!92.026143790849673!black},
y grid style={lightgray!92.026143790849673!black},
legend entries={{BackProp},{ProxProp (cg3)},{ProxProp (cg5)},{ProxProp (cg10)},{ProxProp (exact)}},
legend style={draw=none},
xticklabel style={font=\footnotesize},
legend cell align={left},
legend style={line width=2.0, fill=gray!7},
yticklabel style={font=\footnotesize}
]
\addlegendimage{no markers, color0}
\addlegendimage{no markers, color1}
\addlegendimage{no markers, color2}
\addlegendimage{no markers, color3}
\addlegendimage{no markers, color4}
\addplot [ultra thick, color0]
table {%
0 2.30303915871514
2.80160713195801 1.82629780107074
5.57698273658752 1.6651112847858
8.41791701316833 1.56001142528322
11.2094683647156 1.6086609866884
14.0283279418945 1.42706305848228
16.9474229812622 1.39448691871431
20.0303609371185 1.3879857579867
23.106618642807 1.29451384411918
26.1724627017975 1.30998065736559
29.2532017230988 1.22834987110562
32.3365776538849 1.17962076266607
35.4148833751678 1.11464673015806
38.492390871048 1.12149855030908
41.5599434375763 1.03148577941789
44.6166474819183 0.993953855832418
47.7407612800598 1.00781426628431
50.7990851402283 0.87905450463295
53.8683836460114 0.836888851722081
56.9212222099304 0.825145271089342
59.989874124527 0.727306927575005
63.0463826656342 0.655623251199722
66.0999336242676 0.639317576090495
69.1992347240448 0.528889227575726
72.2894778251648 0.515071266889572
75.379891872406 0.427906752294964
78.4501769542694 0.37654496398237
81.5250008106232 0.411199488242467
84.5846869945526 0.299321324129899
87.6427297592163 0.253434995975759
90.71937084198 0.22471166998148
93.7728490829468 0.237789079712497
96.8235569000244 0.253037561807368
99.894912481308 0.167204022407532
102.965431928635 0.192552308572663
106.028424978256 0.152031081004275
109.0846991539 0.165784892357058
112.152375936508 0.180362060666084
115.211457490921 0.130706705070204
118.270585775375 0.111194670117564
121.360610961914 0.123013379093673
124.455889463425 0.0797373488545418
127.526885271072 0.0850462079462078
130.582374811172 0.118537840826644
133.690487384796 0.09707718036241
136.7523458004 0.0861431931042009
139.805301904678 0.0990064354406463
142.888852596283 0.112416916754511
145.942222833633 0.0798498996429973
149.02388882637 0.090594990303119
152.093945741653 0.0945489556011226
155.174184560776 0.0984268164055215
158.243000984192 0.0913189896278911
161.29957652092 0.085051103722718
164.390043020248 0.0685384716838598
167.479785919189 0.0850127333982123
170.556713342667 0.0766947035574251
173.655479192734 0.0693235403547684
176.709494590759 0.0631841853467955
179.76146197319 0.102244291827083
182.826312541962 0.0816773948984014
185.886902570724 0.0579884931031201
188.953451395035 0.0652685563183493
192.011181116104 0.0964392283724414
195.098617076874 0.0773874489383565
198.151766061783 0.0618358069616887
201.245816230774 0.0778976136197646
204.318851470947 0.074086116730339
207.390568494797 0.0763435177505016
210.446120023727 0.055177615582943
213.523978233337 0.0780759923160076
216.599453449249 0.0680527044667138
219.652366399765 0.092055532336235
222.738651990891 0.065479816434284
225.806044340134 0.0607016299644278
228.871037483215 0.0810137928773959
231.923669338226 0.06374690487153
234.984225511551 0.0675695460496677
238.076774597168 0.0343805300278796
241.155529737473 0.0852606048600541
244.208283185959 0.0548901427537203
247.305730581284 0.0681092219634189
250.369426250458 0.0625491239337458
253.434516191483 0.0614092980614967
256.492159128189 0.0514140410762694
259.587759017944 0.0583537339129382
262.649583101273 0.0590332854953077
265.709500312805 0.0547221495666438
268.784507036209 0.0512664608243439
271.843970298767 0.0599721486162808
274.966011285782 0.0534990188562208
278.04783821106 0.0674813056571616
281.117710828781 0.0418920943410032
284.198012590408 0.047962622768763
287.271510839462 0.0578976255324152
290.340818405151 0.0607633682588736
293.415333271027 0.0525412434505092
296.46733379364 0.0538566199648711
299.560488462448 0.0487125364339186
};
\addplot [ultra thick, color1]
table {%
0 2.30308895375994
8.21554827690125 1.65446132818858
16.673907995224 1.48359440697564
25.1450486183167 1.38602616786957
33.6105337142944 1.2881861699952
42.0687952041626 1.19928237862057
50.5406014919281 1.12200529045529
58.9856204986572 1.05227045615514
67.4554371833801 0.968558493587706
75.9229669570923 0.898367754618327
84.3968212604523 0.821107988225089
92.8462111949921 0.756362900469038
101.311863899231 0.662139523691601
109.779231309891 0.604331546359592
118.242556810379 0.530213640464677
126.707168102264 0.462932740648588
135.181926965714 0.410812931259473
143.65597486496 0.348615672190984
152.118311166763 0.305788951781061
160.572857618332 0.252604829271634
169.055763483047 0.214595197472307
177.521215677261 0.177829329007202
185.988863706589 0.139382360296117
194.463213682175 0.117830156700479
202.927716732025 0.0857456148498588
211.390677928925 0.0645427049448093
219.856781244278 0.0524868350062105
228.339558124542 0.0351585152869423
236.80663394928 0.0250834879568881
245.269932746887 0.0184838843842347
253.754271507263 0.0141538239498105
262.211671590805 0.0115423769379656
270.662837743759 0.00972150907748275
279.139423847198 0.00872214230087896
287.62064242363 0.00758464226188759
296.082925319672 0.00691335886820323
};
\addplot [ultra thick, color2]
table {%
0 2.30288625823127
11.3015723228455 1.60462185144424
22.9527561664581 1.39983485274845
34.620386838913 1.25313382413652
46.2750082015991 1.13212262524499
57.9187955856323 1.02723664310243
69.5814790725708 0.887100534306632
81.2281539440155 0.783470136589474
92.8669204711914 0.660186680157979
104.501347780228 0.565432102150387
116.165598392487 0.461201825075679
127.815102100372 0.358489187558492
139.463114023209 0.294064832230409
151.128698348999 0.225273457831807
162.768900632858 0.163590393380986
174.427403926849 0.128789309908946
186.087549448013 0.0961331327756246
197.756353378296 0.0682628682090176
209.397246360779 0.0458986883362134
221.034046888351 0.0302620469282071
232.694694042206 0.0207004744249086
244.34715294838 0.0111343536267264
255.987144470215 0.00665343485048248
267.660050868988 0.00473132918640557
279.315522432327 0.00397652867250145
290.968179702759 0.00335493262391537
};
\addplot [ultra thick, color3]
table {%
0 2.30253341197968
19.4294092655182 1.53307735125224
38.9520668983459 1.29302298360401
58.4834444522858 1.0982457127836
77.9968528747559 0.893995380401611
97.5058753490448 0.705279644992616
117.025337219238 0.538138543897205
136.544312953949 0.368056487705972
156.069212198257 0.26385682374239
175.586132526398 0.176204272442394
195.120096445084 0.131573138634364
214.62884926796 0.106296145915985
234.139471292496 0.0774215534329414
253.662311077118 0.0750095270574093
273.173390865326 0.0627510632077853
292.694965362549 0.0587704913069804
};
\addplot [ultra thick, color4]
table {%
0 2.30259527895186
19.3000996112823 1.43528225421906
38.187331199646 1.11779054138396
57.5629308223724 0.786926379468706
77.273499250412 0.485072463750839
95.4217846393585 0.261642377575239
114.168922424316 0.130827238659064
132.888016223907 0.0811428907430834
151.750738382339 0.0558163819420669
171.376266956329 0.0474345992836687
192.038754940033 0.0401589810020394
212.240274429321 0.0311777283954951
232.771944046021 0.0324679580931034
252.76126408577 0.0248847088362608
273.905563354492 0.0172563866712153
295.930700778961 0.010575224366039
};
\end{axis}

\end{tikzpicture}

%% file: figures/exact_vs_inexact_prox_mlp_plot_val_acc.tex
\begin{tikzpicture}

\definecolor{color0}{rgb}{0.12156862745098,0.466666666666667,0.705882352941177}
\definecolor{color1}{rgb}{1,0.498039215686275,0.0549019607843137}
\definecolor{color2}{rgb}{0.172549019607843,0.627450980392157,0.172549019607843}
\definecolor{color3}{rgb}{0.83921568627451,0.152941176470588,0.156862745098039}
\definecolor{color4}{rgb}{0.580392156862745,0.403921568627451,0.741176470588235}

\begin{axis}[
title={CIFAR-10, 3072-4000-1000-4000-10 MLP},
xlabel={Epochs},
ylabel={Validation Accuracy},
xmin=-2.45, xmax=51.45,
ymin=0.07457, ymax=0.58123,
tick align=outside,
tick pos=left,
x grid style={lightgray!92.026143790849673!black},
y grid style={lightgray!92.026143790849673!black},
legend entries={{BackProp},{ProxProp (cg3)},{ProxProp (cg5)},{ProxProp (cg10)},{ProxProp (exact)}},
xticklabel style={font=\footnotesize},
legend cell align={left},
legend style={line width=2.0, fill=gray!7},
yticklabel style={font=\footnotesize},
legend style={at={(0.97,0.03)}, anchor=south east, draw=none}
]
\addlegendimage{no markers, color0}
\addlegendimage{no markers, color1}
\addlegendimage{no markers, color2}
\addlegendimage{no markers, color3}
\addlegendimage{no markers, color4}
\addplot [ultra thick, color0]
table {%
0 0.0976
1 0.3188
2 0.4098
3 0.4362
4 0.4602
5 0.4646
6 0.4742
7 0.3404
8 0.5018
9 0.5126
10 0.5224
11 0.5184
12 0.5288
13 0.5358
14 0.5228
15 0.5472
16 0.5398
17 0.5456
18 0.5364
19 0.549
20 0.5514
21 0.534
22 0.5364
23 0.5322
24 0.5346
25 0.515
26 0.5442
27 0.5482
28 0.5326
29 0.5262
30 0.5362
31 0.5354
32 0.5246
33 0.545
34 0.5438
35 0.5406
36 0.5348
37 0.5348
38 0.5386
39 0.527
40 0.5472
41 0.5366
42 0.529
43 0.537
44 0.5462
45 0.538
46 0.535
47 0.5332
48 0.5362
49 0.5328
};
\addplot [ultra thick, color1]
table {%
0 0.105
1 0.4038
2 0.4512
3 0.4774
4 0.5028
5 0.5058
6 0.5238
7 0.5256
8 0.53
9 0.5392
10 0.535
11 0.5398
12 0.5506
13 0.5422
14 0.5428
15 0.54
16 0.5424
17 0.5344
18 0.5396
19 0.5386
20 0.5388
21 0.5352
22 0.543
23 0.543
24 0.5442
25 0.5426
26 0.5394
27 0.5464
28 0.5538
29 0.551
30 0.5518
31 0.5528
32 0.5548
33 0.551
34 0.5522
35 0.5534
36 0.552
37 0.555
38 0.5532
39 0.5518
40 0.5538
41 0.5528
42 0.554
43 0.5534
44 0.5546
45 0.5534
46 0.553
47 0.5542
48 0.5544
49 0.5528
};
\addplot [ultra thick, color2]
table {%
0 0.1018
1 0.4244
2 0.479
3 0.5006
4 0.532
5 0.526
6 0.5422
7 0.5376
8 0.54
9 0.543
10 0.5432
11 0.5412
12 0.5404
13 0.5392
14 0.5368
15 0.5434
16 0.5376
17 0.5412
18 0.5426
19 0.5416
20 0.5494
21 0.5492
22 0.551
23 0.557
24 0.5554
25 0.5572
26 0.5554
27 0.5546
28 0.5552
29 0.5558
30 0.5582
31 0.5556
32 0.5562
33 0.5564
34 0.5558
35 0.5562
36 0.5566
37 0.5572
38 0.5552
39 0.5568
40 0.5572
41 0.5558
42 0.5556
43 0.5574
44 0.5552
45 0.5556
46 0.556
47 0.5562
48 0.5558
49 0.5548
};
\addplot [ultra thick, color3]
table {%
0 0.1386
1 0.4408
2 0.4992
3 0.5212
4 0.5422
5 0.5424
6 0.539
7 0.539
8 0.5352
9 0.5306
10 0.5264
11 0.5382
12 0.5388
13 0.5376
14 0.5392
15 0.53
16 0.5424
17 0.5318
18 0.537
19 0.5314
20 0.54
21 0.5438
22 0.5366
23 0.5438
24 0.5464
25 0.5438
26 0.5502
27 0.5536
28 0.5528
29 0.554
30 0.554
31 0.5544
32 0.5538
33 0.5536
34 0.5544
35 0.5544
36 0.555
37 0.555
38 0.556
39 0.5562
40 0.5562
41 0.5564
42 0.556
43 0.5558
44 0.5566
45 0.5562
46 0.5568
47 0.5574
48 0.5582
49 0.5576
};
\addplot [ultra thick, color4]
table {%
0 0.0986
1 0.4682
2 0.4996
3 0.5268
4 0.5288
5 0.5294
6 0.5298
7 0.53
8 0.5426
9 0.5438
10 0.5342
11 0.5354
12 0.5324
13 0.5362
14 0.5464
15 0.5454
16 0.5486
17 0.5498
18 0.5486
19 0.5494
20 0.5498
21 0.5492
22 0.5486
23 0.5482
24 0.5488
25 0.5488
26 0.549
27 0.5486
28 0.5484
29 0.5472
30 0.548
31 0.5478
32 0.5472
33 0.5482
34 0.546
35 0.5472
36 0.546
37 0.5474
38 0.548
39 0.5486
40 0.545
41 0.5436
42 0.5404
43 0.5398
44 0.5394
45 0.5382
46 0.5372
47 0.5372
48 0.5368
49 0.535
};
\end{axis}

\end{tikzpicture}

%% file: figures/proxprop_vs_sgd_adam_convnet_epochs_plot.tex
\begin{tikzpicture}

\definecolor{color0}{rgb}{0.12156862745098,0.466666666666667,0.705882352941177}
\definecolor{color1}{rgb}{1,0.498039215686275,0.0549019607843137}
\definecolor{color2}{rgb}{0.83921568627451,0.152941176470588,0.156862745098039}

\begin{axis}[
title={CIFAR-10, Convolutional Neural Network},
xlabel={Epochs},
ylabel={Full Batch Loss},
xmin=-2.5, xmax=52.5,
ymin=0.273175439006752, ymax=2.40034148815605,
tick align=outside,
tick pos=left,
x grid style={lightgray!92.026143790849673!black},
y grid style={lightgray!92.026143790849673!black},
legend style={draw=none},
xticklabel style={font=\footnotesize},
legend entries={{Adam + BackProp},{Adam + ProxProp (3 cg)},{Adam + ProxProp (10 cg)}},
legend cell align={left},
legend style={line width=2.0, fill=gray!7},
yticklabel style={font=\footnotesize}
]
\addlegendimage{no markers, color0}
\addlegendimage{no markers, color1}
\addlegendimage{no markers, color2}
\addplot [ultra thick, color0]
table {%
0 2.30268304612901
1 1.71700575616625
2 1.52898537450367
3 1.45309675269657
4 1.39067828920152
5 1.34888476265801
6 1.27649117310842
7 1.24053233597014
8 1.25098515484068
9 1.1831302775277
10 1.20097512801488
11 1.1102712419298
12 1.09531482987934
13 1.08121166494158
14 1.11955722835329
15 1.01022561391195
16 1.0033964269691
17 1.00267048610581
18 0.990202911032571
19 0.95029011103842
20 0.92513985104031
21 0.920417209466298
22 0.926363280084398
23 0.884734435213937
24 0.891663939423031
25 0.899865010049608
26 0.873566477828556
27 0.86544568406211
28 0.839843864573373
29 0.825395409928428
30 0.822769715388616
31 0.84615808195538
32 0.80327250957489
33 0.798036657439338
34 0.801340827676985
35 0.813971053229438
36 0.774303954177433
37 0.757996602190865
38 0.735220128297806
39 0.756572825378842
40 0.755134833521313
41 0.72778944240676
42 0.725984887944327
43 0.735624238517549
44 0.700663130150901
45 0.713355031940672
46 0.704635803567039
47 0.689662490288417
48 0.697021177079942
49 0.674409629238976
50 0.667458183897866
};
\addplot [ultra thick, color1]
table {%
0 2.30365212228563
1 1.62685237195757
2 1.39552887280782
3 1.28010165294011
4 1.2051891234186
5 1.14427435265647
6 1.08677878644731
7 1.02775059673521
8 0.988187989261415
9 0.954050985309813
10 0.918779606289334
11 0.892205938365724
12 0.876057324806849
13 0.856833533445994
14 0.82393257021904
15 0.812020053466161
16 0.798228218158086
17 0.777104021443261
18 0.761402206288444
19 0.736602064636019
20 0.751308396127489
21 0.721182902654012
22 0.705348022778829
23 0.705191187726127
24 0.682113165325589
25 0.672570139831967
26 0.656653838025199
27 0.648778664403492
28 0.632981346713172
29 0.646600388818317
30 0.626808559232288
31 0.623667672607634
32 0.607518140474955
33 0.603655132320192
34 0.584597018029955
35 0.569859901732869
36 0.570346070991622
37 0.562928908732202
38 0.55519837571515
39 0.555986300773091
40 0.537534583277173
41 0.532931253314018
42 0.52947848074966
43 0.531971024473508
44 0.517748387323485
45 0.531206229991383
46 0.509443208575249
47 0.491048315498564
48 0.492975373400582
49 0.487279436323378
50 0.475601843661732
};
\addplot [ultra thick, color2]
table {%
0 2.30205067263709
1 1.542287504673
2 1.33456925418642
3 1.17658570872413
4 1.07327224810918
5 0.998178616497252
6 0.955689428249995
7 0.909759471813838
8 0.876366760995653
9 0.84781798058086
10 0.815521191226112
11 0.792913167344199
12 0.764410424894757
13 0.746156814363268
14 0.728788136773639
15 0.709289996491538
16 0.698653562201394
17 0.674339546759923
18 0.656018988291423
19 0.641036338938607
20 0.629807771576775
21 0.622904578182432
22 0.608008729086982
23 0.593136300643285
24 0.582041286097633
25 0.569753216041459
26 0.564427616529995
27 0.546761583288511
28 0.537777326504389
29 0.526666883296437
30 0.513524591591623
31 0.504274248083432
32 0.503811918695768
33 0.494412767887115
34 0.486345456043879
35 0.476843804452154
36 0.46936547226376
37 0.456871118479305
38 0.45154884159565
39 0.442214912176132
40 0.435791557696131
41 0.433300707075331
42 0.42154984705978
43 0.412852109803094
44 0.40862577425109
45 0.400196430749363
46 0.393498062756326
47 0.390712652934922
48 0.387008541491297
49 0.379987221293979
50 0.369864804877175
};
\end{axis}

\end{tikzpicture}

%% file: figures/proxprop_vs_sgd_adam_convnet_time_plot.tex
\begin{tikzpicture}

\definecolor{color0}{rgb}{0.12156862745098,0.466666666666667,0.705882352941177}
\definecolor{color1}{rgb}{1,0.498039215686275,0.0549019607843137}
\definecolor{color2}{rgb}{0.83921568627451,0.152941176470588,0.156862745098039}

\begin{axis}[
title={CIFAR-10, Convolutional Neural Network},
xlabel={Time [s]},
ylabel={Full Batch Loss},
xmin=-16.4158574938774, xmax=344.733007371426,
ymin=0.19917587455776, ymax=2.40386527693934,
tick align=outside,
tick pos=left,
x grid style={lightgray!92.026143790849673!black},
y grid style={lightgray!92.026143790849673!black},
legend style={draw=none},
xticklabel style={font=\footnotesize},
legend entries={{Adam + BackProp},{Adam + ProxProp (3 cg)},{Adam + ProxProp (10 cg)}},
legend cell align={left},
legend style={line width=2.0, fill=gray!7},
yticklabel style={font=\footnotesize}
]
\addlegendimage{no markers, color0}
\addlegendimage{no markers, color1}
\addlegendimage{no markers, color2}
\addplot [ultra thick, color0]
table {%
0 2.30268304612901
1.89433717727661 1.71700575616625
3.7925271987915 1.52898537450367
5.69851016998291 1.45309675269657
7.59797096252441 1.39067828920152
9.49863386154175 1.34888476265801
11.4323124885559 1.27649117310842
13.3338956832886 1.24053233597014
15.2358980178833 1.25098515484068
17.1382982730865 1.1831302775277
19.0483644008636 1.20097512801488
20.9769518375397 1.1102712419298
22.8913424015045 1.09531482987934
24.8404622077942 1.08121166494158
26.7549877166748 1.11955722835329
28.6698336601257 1.01022561391195
30.5910477638245 1.0033964269691
32.5138049125671 1.00267048610581
34.4292738437653 0.990202911032571
36.3440518379211 0.95029011103842
38.2675034999847 0.92513985104031
40.1823680400848 0.920417209466298
42.0984845161438 0.926363280084398
44.0223863124847 0.884734435213937
45.9379999637604 0.891663939423031
47.8549873828888 0.899865010049608
49.7709636688232 0.873566477828556
51.6948854923248 0.86544568406211
53.6114001274109 0.839843864573373
55.5300552845001 0.825395409928428
57.5140659809113 0.822769715388616
59.5878524780273 0.84615808195538
61.6587901115417 0.80327250957489
63.7472865581512 0.798036657439338
65.8218805789948 0.801340827676985
67.8867704868317 0.813971053229438
69.951141834259 0.774303954177433
72.0228455066681 0.757996602190865
74.1076877117157 0.735220128297806
76.1998193264008 0.756572825378842
78.2721943855286 0.755134833521313
80.3313825130463 0.72778944240676
82.4086804389954 0.725984887944327
84.4894242286682 0.735624238517549
86.5910704135895 0.700663130150901
88.7206211090088 0.713355031940672
90.8157575130463 0.704635803567039
92.9178242683411 0.689662490288417
95.0185332298279 0.697021177079942
97.1175775527954 0.674409629238976
99.231285572052 0.667458183897866
101.344652891159 0.658158834775289
103.450909852982 0.668941985236274
105.557855606079 0.660591977172428
107.662654876709 0.645333438449436
109.764672517776 0.655388035376867
111.894049406052 0.63304288453526
114.014658927917 0.634126301606496
116.138011932373 0.623082996739282
118.260940790176 0.620687111881044
120.383989095688 0.612003280056847
122.513880729675 0.611243961254756
124.637498378754 0.622778049442503
126.760796785355 0.602859848075443
128.891323804855 0.643048228820165
131.014828443527 0.575264976421992
133.13685131073 0.604507329066594
135.266985654831 0.583317235443327
137.417718172073 0.567694454722934
139.540804862976 0.58047892053922
141.665460348129 0.565631747908062
143.795974969864 0.555154192778799
145.918535709381 0.569208778275384
148.041575431824 0.559792699085342
150.171927452087 0.554298571083281
152.301739692688 0.550140207674768
154.425650835037 0.538204216294818
156.549001932144 0.564542620380719
158.701976776123 0.540610313415527
160.825531959534 0.550582324796253
162.949127435684 0.512324852744738
165.079795837402 0.550180062982771
167.203973054886 0.505296803514163
169.326078653336 0.519736927085453
171.449214935303 0.527487520045704
173.579586744308 0.556173746784528
175.703007459641 0.500406001342667
177.84949016571 0.510235400994619
180.006151914597 0.51112224691444
182.128333091736 0.497288056545787
184.251139640808 0.517714247107506
186.380780696869 0.480833743678199
188.533667802811 0.469078368941943
190.657947063446 0.473848158121109
192.781032562256 0.48517145646943
194.91041135788 0.489236484633552
197.034756660461 0.484939045376248
199.1572265625 0.477382874488831
201.287840843201 0.476139844788445
203.411428689957 0.4578675803211
205.561154127121 0.505101075106197
207.684436559677 0.445372184448772
209.815286636353 0.454249826073647
211.938067674637 0.442695118321313
214.061583995819 0.506627263294326
216.191851377487 0.430960915154881
218.31489276886 0.434727798236741
220.439216136932 0.429686448309157
222.570378065109 0.425423249271181
224.701051950455 0.424881175491545
226.828269481659 0.439387326439222
228.973386764526 0.420144142044915
231.103122711182 0.447120052907202
233.22599363327 0.40554733077685
235.37663769722 0.408772643738323
237.506893873215 0.401013551818
239.629700899124 0.43590827981631
241.791138410568 0.412530181474156
243.921428918839 0.421404512723287
246.050953388214 0.382584193348885
248.178231716156 0.395649809969796
250.317961215973 0.390035751130846
252.44734454155 0.417087611556053
254.569424629211 0.40403348075019
256.69144320488 0.378275163637267
258.841067075729 0.405832433700562
260.971524238586 0.406274295846621
263.122860431671 0.380976855423715
265.246382951736 0.379966991146406
267.377053022385 0.372854351335102
269.500223636627 0.350525025195546
271.643395185471 0.377962962124083
273.77313709259 0.381606436438031
275.895341396332 0.367504655321439
278.018393039703 0.366842231485579
280.141193151474 0.34397147960133
282.271218299866 0.341554274492794
284.414425134659 0.353977127538787
286.537534713745 0.352339273691177
288.667993783951 0.346879892216788
290.805509328842 0.392900874548488
292.928861379623 0.332296871145566
295.058816194534 0.329121259848277
297.18901014328 0.331665544377433
299.332471370697 0.318636104464531
301.454491615295 0.352656212449074
303.610754728317 0.337445054120488
305.733472585678 0.334254056546423
307.856061220169 0.328530065218608
309.984244346619 0.312195847431819
312.113196372986 0.334264837702115
314.23493885994 0.321347378691038
316.35736322403 0.312910380297237
318.486872196198 0.323948126037916
320.60807299614 0.301983249187469
322.730516910553 0.299389029211468
324.871690273285 0.316486997074551
326.994924783707 0.304948969682058
};
\addplot [ultra thick, color1]
table {%
0 2.30365212228563
6.21519708633423 1.62685237195757
12.784411907196 1.39552887280782
19.3706915378571 1.28010165294011
25.9456467628479 1.2051891234186
32.5101354122162 1.14427435265647
39.0915117263794 1.08677878644731
45.6586799621582 1.02775059673521
52.226669549942 0.988187989261415
58.7906048297882 0.954050985309813
65.3852636814117 0.918779606289334
71.965993642807 0.892205938365724
78.5362803936005 0.876057324806849
85.1225161552429 0.856833533445994
91.6833148002625 0.82393257021904
98.2481074333191 0.812020053466161
104.842417240143 0.798228218158086
111.431854724884 0.777104021443261
117.994336366653 0.761402206288444
124.564521074295 0.736602064636019
131.15101647377 0.751308396127489
137.718773365021 0.721182902654012
144.284725904465 0.705348022778829
150.864385604858 0.705191187726127
157.432880163193 0.682113165325589
163.999146699905 0.672570139831967
170.567499160767 0.656653838025199
177.151084423065 0.648778664403492
183.711842536926 0.632981346713172
190.276203632355 0.646600388818317
196.860874891281 0.626808559232288
203.427155017853 0.623667672607634
209.992597579956 0.607518140474955
216.57861661911 0.603655132320192
223.163863897324 0.584597018029955
229.726635694504 0.569859901732869
236.298860311508 0.570346070991622
242.886479854584 0.562928908732202
249.454825878143 0.55519837571515
256.023321866989 0.555986300773091
262.615158081055 0.537534583277173
269.183506011963 0.532931253314018
275.759853124619 0.52947848074966
282.323591709137 0.531971024473508
288.897578716278 0.517748387323485
295.467103004456 0.531206229991383
302.031626224518 0.509443208575249
308.605200052261 0.491048315498564
315.181202888489 0.492975373400582
321.742097139359 0.487279436323378
328.317149877548 0.475601843661732
};
\addplot [ultra thick, color2]
table {%
0 2.30205067263709
14.0069394111633 1.542287504673
28.3923609256744 1.33456925418642
42.7985503673553 1.17658570872413
57.1809704303741 1.07327224810918
71.5618493556976 0.998178616497252
85.9542090892792 0.955689428249995
100.334239244461 0.909759471813838
114.716428279877 0.876366760995653
129.092986822128 0.84781798058086
143.490203380585 0.815521191226112
157.879465818405 0.792913167344199
172.253810167313 0.764410424894757
186.638559341431 0.746156814363268
201.024842262268 0.728788136773639
215.401744365692 0.709289996491538
229.795499563217 0.698653562201394
244.189144611359 0.674339546759923
258.56639289856 0.656018988291423
272.944904088974 0.641036338938607
287.343618392944 0.629807771576775
301.723384857178 0.622904578182432
316.114305973053 0.608008729086982
};
\end{axis}

\end{tikzpicture}

%% file: figures/proxprop_vs_sgd_adam_convnet_epochs_plot_val.tex
\begin{tikzpicture}

\definecolor{color0}{rgb}{0.12156862745098,0.466666666666667,0.705882352941177}
\definecolor{color1}{rgb}{1,0.498039215686275,0.0549019607843137}
\definecolor{color2}{rgb}{0.83921568627451,0.152941176470588,0.156862745098039}

\begin{axis}[
title={CIFAR-10, Convolutional Neural Network},
xlabel={Epochs},
ylabel={Validation Accuracy},
xmin=-2.5, xmax=52.5,
ymin=0.05615, ymax=0.76125,
tick align=outside,
tick pos=left,
x grid style={lightgray!92.026143790849673!black},
y grid style={lightgray!92.026143790849673!black},
xticklabel style={font=\footnotesize},
legend entries={{Adam + BackProp},{Adam + ProxProp (3 cg)},{Adam + ProxProp (10 cg)}},
legend cell align={left},
legend style={line width=2.0, fill=gray!7},
yticklabel style={font=\footnotesize},
legend style={at={(0.97,0.03)}, anchor=south east, draw=none}
]
\addlegendimage{no markers, color0}
\addlegendimage{no markers, color1}
\addlegendimage{no markers, color2}
\addplot [ultra thick, color0]
table {%
0 0.0954
1 0.3802
2 0.4418
3 0.4838
4 0.4992
5 0.5248
6 0.5498
7 0.5602
8 0.5524
9 0.5798
10 0.563
11 0.6052
12 0.6216
13 0.6092
14 0.5938
15 0.651
16 0.6402
17 0.6386
18 0.6468
19 0.664
20 0.665
21 0.6716
22 0.6632
23 0.6782
24 0.6694
25 0.6696
26 0.6772
27 0.6752
28 0.6832
29 0.6922
30 0.6934
31 0.6786
32 0.694
33 0.6958
34 0.6972
35 0.6864
36 0.7068
37 0.7084
38 0.7094
39 0.6982
40 0.6958
41 0.708
42 0.7074
43 0.702
44 0.7098
45 0.7022
46 0.7116
47 0.711
48 0.7118
49 0.7072
50 0.7112
};
\addplot [ultra thick, color1]
table {%
0 0.0882
1 0.4092
2 0.488
3 0.5386
4 0.565
5 0.5904
6 0.6136
7 0.6348
8 0.6458
9 0.6534
10 0.668
11 0.6784
12 0.6754
13 0.6752
14 0.6924
15 0.6886
16 0.6904
17 0.6982
18 0.7014
19 0.707
20 0.6992
21 0.7042
22 0.7096
23 0.7082
24 0.7148
25 0.7136
26 0.7188
27 0.7158
28 0.7242
29 0.7164
30 0.7184
31 0.7212
32 0.724
33 0.7184
34 0.7274
35 0.7272
36 0.7252
37 0.7258
38 0.7282
39 0.7244
40 0.7292
41 0.7284
42 0.7236
43 0.7258
44 0.7238
45 0.7214
46 0.7254
47 0.729
48 0.7266
49 0.724
50 0.7236
};
\addplot [ultra thick, color2]
table {%
0 0.0988
1 0.4368
2 0.5192
3 0.5848
4 0.6242
5 0.6446
6 0.6632
7 0.6714
8 0.6818
9 0.6894
10 0.6998
11 0.7004
12 0.7064
13 0.7146
14 0.7128
15 0.7174
16 0.7188
17 0.7214
18 0.7254
19 0.7268
20 0.7274
21 0.7264
22 0.7266
23 0.7264
24 0.7274
25 0.7234
26 0.7236
27 0.727
28 0.7264
29 0.725
30 0.7244
31 0.728
32 0.7244
33 0.7258
34 0.7266
35 0.724
36 0.7208
37 0.7202
38 0.7242
39 0.7224
40 0.7176
41 0.7214
42 0.721
43 0.7148
44 0.7196
45 0.716
46 0.7174
47 0.7134
48 0.7144
49 0.7126
50 0.7126
};
\end{axis}

\end{tikzpicture}

%% file: figures/proxprop_vs_sgd_adam_convnet_time_plot_val.tex
\begin{tikzpicture}

\definecolor{color0}{rgb}{0.12156862745098,0.466666666666667,0.705882352941177}
\definecolor{color1}{rgb}{1,0.498039215686275,0.0549019607843137}
\definecolor{color2}{rgb}{0.83921568627451,0.152941176470588,0.156862745098039}

\begin{axis}[
title={CIFAR-10, Convolutional Neural Network},
xlabel={Time [s]},
ylabel={Validation Accuracy},
xmin=-16.4158574938774, xmax=344.733007371426,
ymin=0.05615, ymax=0.76125,
tick align=outside,
tick pos=left,
x grid style={lightgray!92.026143790849673!black},
y grid style={lightgray!92.026143790849673!black},
xticklabel style={font=\footnotesize},
legend entries={{Adam + BackProp},{Adam + ProxProp (3 cg)},{Adam + ProxProp (10 cg)}},
legend cell align={left},
legend style={line width=2.0, fill=gray!7},
yticklabel style={font=\footnotesize},
legend style={at={(0.97,0.03)}, anchor=south east, draw=none}
]
\addlegendimage{no markers, color0}
\addlegendimage{no markers, color1}
\addlegendimage{no markers, color2}
\addplot [ultra thick, color0]
table {%
0 0.0954
1.89433717727661 0.3802
3.7925271987915 0.4418
5.69851016998291 0.4838
7.59797096252441 0.4992
9.49863386154175 0.5248
11.4323124885559 0.5498
13.3338956832886 0.5602
15.2358980178833 0.5524
17.1382982730865 0.5798
19.0483644008636 0.563
20.9769518375397 0.6052
22.8913424015045 0.6216
24.8404622077942 0.6092
26.7549877166748 0.5938
28.6698336601257 0.651
30.5910477638245 0.6402
32.5138049125671 0.6386
34.4292738437653 0.6468
36.3440518379211 0.664
38.2675034999847 0.665
40.1823680400848 0.6716
42.0984845161438 0.6632
44.0223863124847 0.6782
45.9379999637604 0.6694
47.8549873828888 0.6696
49.7709636688232 0.6772
51.6948854923248 0.6752
53.6114001274109 0.6832
55.5300552845001 0.6922
57.5140659809113 0.6934
59.5878524780273 0.6786
61.6587901115417 0.694
63.7472865581512 0.6958
65.8218805789948 0.6972
67.8867704868317 0.6864
69.951141834259 0.7068
72.0228455066681 0.7084
74.1076877117157 0.7094
76.1998193264008 0.6982
78.2721943855286 0.6958
80.3313825130463 0.708
82.4086804389954 0.7074
84.4894242286682 0.702
86.5910704135895 0.7098
88.7206211090088 0.7022
90.8157575130463 0.7116
92.9178242683411 0.711
95.0185332298279 0.7118
97.1175775527954 0.7072
99.231285572052 0.7112
101.344652891159 0.7108
103.450909852982 0.7112
105.557855606079 0.7132
107.662654876709 0.7104
109.764672517776 0.708
111.894049406052 0.7174
114.014658927917 0.7088
116.138011932373 0.7178
118.260940790176 0.711
120.383989095688 0.7124
122.513880729675 0.7134
124.637498378754 0.7062
126.760796785355 0.7154
128.891323804855 0.7006
131.014828443527 0.716
133.13685131073 0.7104
135.266985654831 0.7112
137.417718172073 0.7156
139.540804862976 0.7082
141.665460348129 0.713
143.795974969864 0.7176
145.918535709381 0.7018
148.041575431824 0.7118
150.171927452087 0.7098
152.301739692688 0.714
154.425650835037 0.7122
156.549001932144 0.7116
158.701976776123 0.7068
160.825531959534 0.7048
162.949127435684 0.7144
165.079795837402 0.7064
167.203973054886 0.7106
169.326078653336 0.7132
171.449214935303 0.7082
173.579586744308 0.6914
175.703007459641 0.7098
177.84949016571 0.7092
180.006151914597 0.7034
182.128333091736 0.7084
184.251139640808 0.6996
186.380780696869 0.7092
188.533667802811 0.7122
190.657947063446 0.7084
192.781032562256 0.7048
194.91041135788 0.7118
197.034756660461 0.6998
199.1572265625 0.7026
201.287840843201 0.7064
203.411428689957 0.7064
205.561154127121 0.7002
207.684436559677 0.7102
209.815286636353 0.7052
211.938067674637 0.7096
214.061583995819 0.6968
216.191851377487 0.7116
218.31489276886 0.7042
220.439216136932 0.7056
222.570378065109 0.7072
224.701051950455 0.6992
226.828269481659 0.6976
228.973386764526 0.6982
231.103122711182 0.699
233.22599363327 0.7048
235.37663769722 0.7008
237.506893873215 0.7048
239.629700899124 0.7004
241.791138410568 0.6966
243.921428918839 0.6992
246.050953388214 0.7036
248.178231716156 0.6964
250.317961215973 0.7016
252.44734454155 0.6898
254.569424629211 0.695
256.69144320488 0.6996
258.841067075729 0.6964
260.971524238586 0.6914
263.122860431671 0.6986
265.246382951736 0.6942
267.377053022385 0.696
269.500223636627 0.6968
271.643395185471 0.6948
273.77313709259 0.6944
275.895341396332 0.6996
278.018393039703 0.6896
280.141193151474 0.6954
282.271218299866 0.6966
284.414425134659 0.6926
286.537534713745 0.6946
288.667993783951 0.6954
290.805509328842 0.6836
292.928861379623 0.6958
295.058816194534 0.6974
297.18901014328 0.7004
299.332471370697 0.693
301.454491615295 0.6842
303.610754728317 0.6954
305.733472585678 0.6934
307.856061220169 0.6888
309.984244346619 0.6952
312.113196372986 0.6928
314.23493885994 0.69
316.35736322403 0.6928
318.486872196198 0.6896
320.60807299614 0.6952
322.730516910553 0.6948
324.871690273285 0.6934
326.994924783707 0.691
};
\addplot [ultra thick, color1]
table {%
0 0.0882
6.21519708633423 0.4092
12.784411907196 0.488
19.3706915378571 0.5386
25.9456467628479 0.565
32.5101354122162 0.5904
39.0915117263794 0.6136
45.6586799621582 0.6348
52.226669549942 0.6458
58.7906048297882 0.6534
65.3852636814117 0.668
71.965993642807 0.6784
78.5362803936005 0.6754
85.1225161552429 0.6752
91.6833148002625 0.6924
98.2481074333191 0.6886
104.842417240143 0.6904
111.431854724884 0.6982
117.994336366653 0.7014
124.564521074295 0.707
131.15101647377 0.6992
137.718773365021 0.7042
144.284725904465 0.7096
150.864385604858 0.7082
157.432880163193 0.7148
163.999146699905 0.7136
170.567499160767 0.7188
177.151084423065 0.7158
183.711842536926 0.7242
190.276203632355 0.7164
196.860874891281 0.7184
203.427155017853 0.7212
209.992597579956 0.724
216.57861661911 0.7184
223.163863897324 0.7274
229.726635694504 0.7272
236.298860311508 0.7252
242.886479854584 0.7258
249.454825878143 0.7282
256.023321866989 0.7244
262.615158081055 0.7292
269.183506011963 0.7284
275.759853124619 0.7236
282.323591709137 0.7258
288.897578716278 0.7238
295.467103004456 0.7214
302.031626224518 0.7254
308.605200052261 0.729
315.181202888489 0.7266
321.742097139359 0.724
328.317149877548 0.7236
};
\addplot [ultra thick, color2]
table {%
0 0.0988
14.0069394111633 0.4368
28.3923609256744 0.5192
42.7985503673553 0.5848
57.1809704303741 0.6242
71.5618493556976 0.6446
85.9542090892792 0.6632
100.334239244461 0.6714
114.716428279877 0.6818
129.092986822128 0.6894
143.490203380585 0.6998
157.879465818405 0.7004
172.253810167313 0.7064
186.638559341431 0.7146
201.024842262268 0.7128
215.401744365692 0.7174
229.795499563217 0.7188
244.189144611359 0.7214
258.56639289856 0.7254
272.944904088974 0.7268
287.343618392944 0.7274
301.723384857178 0.7264
316.114305973053 0.7266
};
\end{axis}

\end{tikzpicture}

%% file: supp_content.tex
\setcounter{section}{0}
\renewcommand\thesection{\Alph{section}}
\section{Theoretical results}
\label{sec:proofs}
\begin{proof}[Proof of Proposition~1]
We first take a gradient step on 
\begin{align}
\label{eq:penaltyFunctionSupp}
E(\vec{\theta}, \vec{a}, \vec{z}) &= \L_y(\phi(\theta_{L-1}, a_{L-2}))  
 + \sum_{l=1}^{L-2} \frac{\gamma}{2} \|\sigma(z_{l}) - a_l\|^2 
+ \frac{\rho}{2}\|\phi(\theta_l, a_{l-1}) - z_l \|^2,
\end{align} 
with respect to $(\theta_{L-1}, a_{L-2})$. The gradient step with respect to $\theta_{L-1}$ is the same as in the gradient descent update, 
\begin{equation}
\begin{aligned}
    \vec{\theta}^{k+1} = \vec{\theta}^{k} - \tau \nabla J(\vec{\theta}^{k}; X, y),
\label{eq:gradient_descentSupp}
\end{aligned}
\end{equation}
since $J$ depends on $\theta_{L-1}$ only via $\L_y \circ \phi$.

The gradient descent step on $a_{L-2}$ in \hyperref[alg1:gradient_a]{\textcirclednice{a}} yields 
\begin{align}
\label{eq:last_a_updateSupp}
 a_{L-2}^{k+\nicefrac{1}{2}} &= a_{L-2}^{k} - \tau \nabla_a\phi(\theta_{L-1}^{k}, a_{L-2}^{k}) 
 \cdot  \nabla_\phi \L_y(\phi(\theta_{L-1}^{k}, a_{L-2}^{k})),
\end{align}
where we use $ a_{L-2}^{k+\nicefrac{1}{2}}$ to denote the updated variable $a_{L-2}$ before the forward pass of the next iteration. To keep the presentation as clear as possible we slightly abused the notation of a right multiplication with $\nabla_a\phi(\theta_{L-1}^{k}, a_{L-2}^{k})$: While this notation is exact in the case of fully connected layers, it represents the application of the corresponding linear operator in the more general case, e.g. for convolutions.

For all layers $l\leq L-2$ note that due to the forward pass in Algorithm~1 we have 
\begin{equation}
	\sigma(z_l^{k}) = a_l^{k}, \quad \phi(\theta_l^{k},a_{l-1}^{k}) = z_l^{k}
\end{equation}
and we therefore get the following update equations in the gradient step \hyperref[alg1:gradient_b]{\textcirclednice{b}}
\begin{align}
\label{eq:z_updateSupp}
\begin{split}
z_l^{k+\nicefrac{1}{2}} = z_l^{k} - \tau\gamma \nabla\sigma (z_l^{k}) \left(\sigma(z_l^{k}) - a_l^{k+\nicefrac{1}{2}} \right) 
= z_l^{k} - \nabla\sigma (z_l^{k}) \left( a_l^{k} - a_l^{k+\nicefrac{1}{2}} \right), 
\end{split}
\end{align}
and in the gradient step \hyperref[alg1:gradient_c]{\textcirclednice{c}} w.r.t. $a_{l-1}$,
\begin{align}
\label{eq:a_update_proofSupp}
\begin{split}
  a_{l-1}^{k+\nicefrac{1}{2}} &= a_{l-1}^{k} - \tau\rho \nabla_{a}\phi(\theta_l^{k}, a_{l-1}^{k})
 \cdot \left( \phi(\theta_l^{k}, a_{l-1}^{k}) - z_l^{k+\nicefrac{1}{2}} \right) \\
&= a_{l-1}^{k} - \nabla_{a}\phi(\theta_l^{k}, a_{l-1}^{k}) 
 \cdot \left( z_l^{k} - z_l^{k+\nicefrac{1}{2}} \right).
\end{split}
\end{align}

Equations \eqref{eq:z_updateSupp} and \eqref{eq:a_update_proofSupp} can be combined to obtain:
\begin{align}
 z_l^{k}  - z_l^{k+\nicefrac{1}{2}} =& \nabla\sigma (z_l^{k})\nabla_{a}\phi(\theta_{l+1}^{k}, a_{l}^{k}) 
 \cdot \left( z_{l+1}^{k} - z_{l+1}^{k+\nicefrac{1}{2}} \right).
\end{align}
The above formula allows us to backtrack the differences of the old $z_l^{k}$ and the updated $z_l^{k+\nicefrac{1}{2}}$ up to layer $L-2$, where we can use equations \eqref{eq:z_updateSupp} and \eqref{eq:last_a_updateSupp} to relate the difference to the loss. 
Altogether, we obtain
\begin{equation}
\begin{aligned}
 z_l^{k}  - z_l^{k+\nicefrac{1}{2}} = 
   \tau \left(\prod_{q=l}^{L-2}\nabla\sigma (z_q^{k})\nabla_{a}\phi(\theta_{q+1}^{k}, a_{q}^{k})\right)
 \cdot  \nabla_\phi \L_y(\phi(\theta_{L-1}^{k}, a_{L-2}^{k})). 
\end{aligned}
\label{eq:above_formulaSupp}
\end{equation}
By inserting \eqref{eq:above_formulaSupp} into the gradient descent update equation with respect to $\theta_l$ in \hyperref[alg1:gradient_c]{\textcirclednice{c}} ,
\begin{align}
\label{eq:thetaUpdate}
\begin{split}
  \theta^{k+1} = \theta^{k} - \nabla_{\theta}\phi(\theta_l^{k}, a_{l-1}^{k}) 
 \cdot \left( z_l^{k} - z_l^{k+\nicefrac{1}{2}} \right) ,
\end{split}
\end{align}
we obtain the chain rule for update \eqref{eq:gradient_descentSupp}.
\end{proof}

\begin{proof}[Proof of Proposition~2]
Since only the updates for $\theta_l$, $l=1, \hdots,L-2$, are performed implicitly, one can replicate the proof of Proposition~1 exactly up to equation \eqref{eq:above_formulaSupp}. Let us denote the right hand side of \eqref{eq:above_formulaSupp} by $g_l^k$, i.e. $z_l^{k+\nicefrac{1}{2}}=z^k_l - g_l^k$ and note that
\begin{equation}
\tau \nabla_{\theta_l} J(\vec{\theta}^k; X,y) = \nabla_\theta \phi(\cdot, a_{l-1}^k) \cdot g_l^k 
\end{equation}
holds by the chain rule (as seen in \eqref{eq:thetaUpdate}). 
We have eliminated the dependence of $\nabla_\theta \phi(\theta_l^k, a_{l-1}^k)$ on $\theta_l^k$ and wrote $\nabla_\theta \phi(\cdot, a_{l-1}^k)$ instead, because we assume $\phi$ to be linear in $\theta$ such that $\nabla_\theta \phi$ does not depend on the point $\theta$ where the gradient is evaluated anymore. 

We now rewrite the ProxProp update equation of the parameters $\theta$ as follows
\begin{equation}
\begin{aligned}
  \theta_{l}^{k+1} =&~ \argmin_{\theta} ~ \frac{1}{2} \norm{\phi(\theta, a_{l-1}^k) - z_l^{k+\nicefrac{1}{2}}}^2 + \frac{1}{2 \tau_\theta} \norm{\theta - \theta_{l}^{k}}^2 \; \\
  =&~ \argmin_{\theta} ~ \frac{1}{2} \norm{\phi(\theta, a_{l-1}^k) - (z_l^{k} - g_l^k)}^2 + \frac{1}{2 \tau_\theta} \norm{\theta - \theta_{l}^{k}}^2 \; \\
    =&~ \argmin_{\theta} ~ \frac{1}{2} \norm{\phi(\theta, a_{l-1}^k) - (\phi(\theta^k, a_{l-1}^k) - g_l^k)}^2 + \frac{1}{2 \tau_\theta} \norm{\theta - \theta_{l}^{k}}^2 \; \\
        =&~ \argmin_{\theta} ~ \frac{1}{2} \norm{\phi(\theta-\theta^k, a_{l-1}^k) + g_l^k}^2 + \frac{1}{2 \tau_\theta} \norm{\theta - \theta_{l}^{k}}^2, \;
\end{aligned}
\end{equation}
where we have used that $\phi$ is linear in $\theta$. The optimality condition yields
\begin{align}
0 =&~ \nabla \phi(\cdot,a_{l-1}^k) (\phi(\theta_{l}^{k+1}-\theta^k, a_{l-1}^k) + g_l^k) + \frac{1}{\tau_\theta}(\theta_{l}^{k+1} - \theta_{l}^{k})
\end{align}
Again, due to the linearity of $\phi$ in $\theta$, one has
\begin{equation} 
  \phi(\theta,a_{l-1}^k) = (\nabla \phi(\cdot,a_{l-1}^k))^*(\theta), 
\end{equation}
where $^*$, denotes the adjoint of a linear operator.
We conclude
\begin{align}
&~0 = \nabla \phi(\cdot,a_{l-1}^k) (\nabla \phi(\cdot,a_{l-1}^k))^*(\theta_{l}^{k+1}-\theta^k) + \nabla \phi(\cdot,a_{l-1}^k) g_l^k + \frac{1}{\tau_\theta}(\theta_{l}^{k+1} - \theta_{l}^{k}),\nonumber\\
\label{eq:helper}
&~\Rightarrow \left(\frac{1}{\tau_\theta} I + \nabla \phi(\cdot,a_{l-1}^k) (\nabla \phi(\cdot,a_{l-1}^k))^*\right)(\theta_{l}^{k+1} - \theta_{l}^{k}) = -\nabla \phi(\cdot,a_{l-1}^k) g_l^k = -\tau \nabla_{\theta_l} J(\vec{\theta}^k; X,y),
\end{align}
which yields the assertion. 
\end{proof}
\begin{proof}[Proof of Proposition~3]
Under the assumption that $\theta^k$ converges, $\theta^k \rightarrow \hat \theta$, one finds that $a_l^k \rightarrow \hat a_l$ and $z_l^k\rightarrow \hat z_l=\phi(\hat \theta_l, \hat a_{l-1})$ converge to the respective activations of the parameters $\hat \theta$ due to the forward pass and the continuity of the network. 
As we assume $J(\cdot; X,y)$ to be continuously differentiable, we deduce from \eqref{eq:helper} that $\lim_{k\to \infty} \nabla_{\theta_l} J(\vec{\theta}^k; X,y) = 0$ for all $l=1,\dots, L-2$. The parameters of the last layer $\theta_{L-1}$ are treated explicitly anyways, such that the above equation also holds for $l=L-1$, which then yields the assertion. 
\end{proof}

\begin{proof}[Proof of Proposition~4]
As the matrices 
\begin{align}
M_l^k := \frac{1}{\tau_\theta}I + (\nabla \phi(\cdot,a_{l-1}^k)) (\nabla \phi(\cdot,a_{l-1}^k))^*
\end{align}
(with the convention $M_{L-1}^k=I$) are positive definite, so are their inverses, and the claim that $\vec{\theta}^{k+1}-\vec{\theta}^k$ is a descent direction is immediate,
\begin{align}
\langle \theta_l^{k+1}- \theta_l^k, -\nabla_{\theta_l}J(\vec{\theta}^k; Y,x) \rangle = \tau \langle (M_l^k)^{-1} \nabla_{\theta_l}J(\vec{\theta}^k; Y,x), \nabla_{\theta_l}J(\vec{\theta}^k; Y,x) \rangle.
\end{align}

We still have to guarantee that this update direction does not become orthogonal to the gradient in the limit $k\to \infty$.
The largest eigenvalue of $(M_l^k)^{-1}$ is bounded from above by $\tau_\theta$. If the $a_{l-1}^k$ remain bounded, then so does $\nabla \phi(\cdot,a_{l-1}^k)$ and the largest eigenvalue of $\nabla
\phi(\cdot,a_{l-1}^k)\nabla \phi(\cdot,a_{l-1}^k)^*$ must be bounded by some constant $\tilde{c}$. Therefore, the smallest eigenvalue of $(M_l^k)^{-1}$ must remain bounded from from below by
$(\frac{1}{\tau_\theta} + \tilde{c})^{-1}$. Abbreviating $v = \nabla_{\theta_l}J(\vec{\theta}^k; Y,x)$, it follows that
\begin{equation}
\begin{aligned}
\cos(\alpha^k) =&~ \frac{\tau \langle (M_l^k)^{-1} v,v \rangle }{\tau \|(M_l^k)^{-1} v\| \|v\|}\\
\geq &~ \frac{\lambda_{\min}((M_l^k)^{-1})\|v\|^2}{\|(M_l^k)^{-1} v\| \|v\|}\\
\geq &~ \frac{\lambda_{\min}((M_l^k)^{-1})}{\lambda_{\max}((M_l^k)^{-1})}
\end{aligned}
\end{equation}
which yields the assertion. 
\end{proof}

\begin{proof}[Proof of Proposition~5]
According to \cite[p. 109, Thm. 5.3]{Nocedal2006} and \cite[p. 106, Thm. 5.2]{Nocedal2006} the $k$-th iteration $x_k$ of the CG method for solving a linear system $Ax=b$ with starting point $x_0=0$ meets
\begin{align}
x_k = \arg \min_{x \in \text{span}(b, Ab, \hdots, A^{k-1}b)}\frac{1}{2}\langle x, Ax \rangle - \langle b, x\rangle,
\end{align}
i.e. is optimizing over an order-$k$ Krylov subspace.
The starting point $x_0=0$ can be chosen without loss of generality.
Suppose the starting point is $\tilde x_0 \neq 0$, then one can optimize the variable $x = \tilde x - \tilde x_0$ with a starting point $x_0 = 0$ and $b = \tilde b + A \tilde x_0$.

We will assume that the CG iteration has not converged yet as the claim for a fully converged CG iteration immediately follow from Proposition 4. Writing the vectors $b, Ab, \hdots, A^{k-1}b$ as columns of a matrix $\mathcal{K}_k$, the condition $x \in \text{span}(b, Ab, \hdots, A^{k-1}b)$ can equivalently be expressed as $x = \mathcal{K}_k \alpha$ for some $\alpha \in \mathbb{R}^k$. In terms of $\alpha$ our minimization problem becomes
\begin{align}
x_k = \mathcal{K}_k \alpha = \arg \min_{\alpha \in \mathbb{R}^k}\frac{1}{2}\langle \alpha , (\mathcal{K}_k)^TA\mathcal{K}_k \alpha \rangle - \langle (\mathcal{K}_k)^T b, \alpha \rangle,
\end{align}
leading to the optimality condition 
\begin{equation}
\begin{aligned}
0 =&~ (\mathcal{K}_k)^TA\mathcal{K}_k \alpha - (\mathcal{K}_k)^T b ,\\
\Rightarrow x_k =&~ \mathcal{K}_k((\mathcal{K}_k)^TA\mathcal{K}_k)^{-1}(\mathcal{K}_k)^T b.
\end{aligned}
\end{equation}
Note that $A$ is symmetric positive definite and can therefore be written as $\sqrt{A}^T \sqrt{A}$, leading to 
\begin{align}
(\mathcal{K}_k)^TA\mathcal{K}_k = (\sqrt{A} \mathcal{K}_k)^T(\sqrt{A}\mathcal{K}_k)
\end{align}
being symmetric positive definite. Hence, the matrix $((\mathcal{K}_k)^TA\mathcal{K}_k)^{-1}$ is positive definite, too, and
\begin{equation}
\begin{aligned}
\langle x_k, b \rangle =&~ \langle \mathcal{K}_k((\mathcal{K}_k)^TA\mathcal{K}_k)^{-1}(\mathcal{K}_k)^T b, b \rangle \\
=&~ \langle ((\mathcal{K}_k)^TA\mathcal{K}_k)^{-1}(\mathcal{K}_k)^T b, (\mathcal{K}_k)^T b \rangle > 0.
\end{aligned}
\end{equation}
Note that $(\mathcal{K}_k)^T b $ is nonzero if $b$ is nonzero, as $\|b\|^2$ is its first entry. 

To translate the general analysis of the CG iteration to our specific case, using any number of CG iterations we find that an approximate solution $\tilde{v}_l^k$ of 
\begin{align}
M_l^k v_l^k = -\nabla_{\theta_l}J(\vec{\theta}^k;X,y)
\end{align}
leads to 
$$\langle \tilde{v}_l^k, -\nabla_{\theta_l}J(\vec{\theta}^k;X,y) \rangle >0,$$
i.e., to $\tilde{v}_l^k$ being a descent direction. 
\end{proof}

\section{Proximal operator for linear transfer functions}
\label{sec:prox_linear}
In order to update the parameters $\theta_l$ of the linear transfer function, we have to solve the problem \eqref{eq:wba_update},
\begin{equation}
    \theta^{k+1} = \argmin_{\theta} ~ \frac{1}{2} \norm{\phi(\theta, a^k) - z^{k+\nicefrac{1}{2}}}^2 + \frac{1}{2 \tau_\theta} \norm{\theta - \theta^{k}}^2 . 
\end{equation}

Since we assume that $\phi$ is linear in $\theta$ for a fixed $a^k$, there exists a matrix $A^k$ such that 
\begin{equation}
    \text{vec}(\theta^{k+1}) = \argmin_{\theta} ~ \frac{1}{2} \norm{A^k\text{vec}(\theta) - \text{vec}(z^{k+\nicefrac{1}{2}})}^2 + \frac{1}{2 \tau_\theta} \norm{\text{vec}(\theta) - \text{vec}(\theta^{k})}^2,
\end{equation}
and the optimality condition yields
\begin{equation}
    \text{vec}(\theta^{k+1}) = (I + \tau_\theta (A^k)^TA^k)^{-1}(\text{vec}(\theta^{k}) + (A^k)^T\text{vec}(z^{k+\nicefrac{1}{2}})).
\end{equation}
In the main paper we sometimes use the more abstract but also more concise notion of $\nabla \phi(\cdot, a^k)$, which represents the linear operator
\begin{equation}
  \nabla \phi(\cdot, a^k)(Y) = \text{vec}^{-1}((A^k)^T\text{vec}(Y)).
\end{equation}

To also make the above more specific, consider the example of $\phi(\theta, a^k) = \theta a^k$. In this case the variable $\theta$ may remain in a matrix form and the solution of the proximal mapping becomes
\begin{equation}
  \begin{aligned}
    \theta^{k+1} = \left( z^{k+\nicefrac{1}{2}} \, (a^k)^\top  + \frac{1}{\tau_\theta} \theta^k \right) \left( a^k (a^k)^\top + \frac{1}{\tau_\theta} I \right)^{-1} . 
  \end{aligned}
  \label{eq:wb_updateSupp}
\end{equation}
Since $a^k \in \bbR^{n \times N}$ for some layer size $n$ and batch size $N$, the size of the linear system is independent of the batch size.